\documentclass[11pt]{article} 

\usepackage[vmargin=1in,hmargin=1in,centering,letterpaper]{geometry}
\usepackage{times}
\usepackage{caption}
\usepackage{color}
\usepackage{amsmath,amssymb,amsfonts,bm,url,dsfont,amsthm}
\usepackage{natbib}

\usepackage[vlined,ruled,linesnumbered]{algorithm2e}

\def\E{\mathbb{E}}

\def\1{\mathbf{1}}
\def\P{\mathbb{P}}
\def\R{\mathbb{R}}

\newcommand{\Unif}{\mathrm{Unif }}

\newcommand{\calW}{\mathcal{W}}
\newcommand{\cardW}[1]{\left|#1\right|_{\calW}}

\newcommand{\Prt}[1]{\Pr_{\cdot | t}\left[#1\right]}
\newcommand{\Expt}[1]{\Exp_{\cdot | t}\left[#1\right]}

\newcommand{\calL}{\mathcal{L}}

\newcommand{\Err}{\mathrm{Err }}
\newcommand{\Siandj}{S_{-i\wedge j}}
\newcommand{\Siorj}{S_{-i\vee j}}
\newcommand{\Siandc}{S_{-i\wedge c}}
\newcommand{\Siorc}{S_{-i\vee c}}
\newcommand{\Splus}{S^{+}}
\newcommand{\Sell}{S^{\ell}}
\newcommand{\calS}{\mathcal{S}}

\newcommand{\I}{\mathbb{I}}
\newcommand{\Exp}{\mathbb{E}}

\newcommand{\Var}{\mathrm{Var}}

\renewcommand{\Pr}{\mathbb{P}}

\newtheorem{thm}{Theorem}[section]
\newtheorem{theorem}{Theorem}[section]
\newtheorem{claim}[thm]{Claim}
\newtheorem{lem}[thm]{Lemma}
\newtheorem{lemma}[thm]{Lemma}

\newtheorem{defn}{Definition}
\newtheorem{prop}[thm]{Proposition}
\newtheorem{cor}[thm]{Corollary}

\newtheorem{rem}{Remark}[section]

\usepackage{mathtools}
\DeclarePairedDelimiter\ceil{\lceil}{\rceil}
\DeclarePairedDelimiter\floor{\lfloor}{\rfloor}

\title{Best-of-K Bandits}
\author{\normalsize {Max Simchowitz} \texttt{msimchow@eecs.berkeley.edu}\\
\normalsize{Kevin Jamieson} \texttt{kjamieson@eecs.berkeley.edu}\\
\normalsize{Ben Recht} \texttt{brecht@eecs.berkeley.edu}\\
 \normalsize University of California, Berkeley, CA 94720 USA
}

\begin{document}

\maketitle

\begin{abstract} 
This paper studies the Best-of-K Bandit game: At each time the player chooses a subset S among all N-choose-K possible options and observes reward max(X(i) : i in S) where X is a random vector drawn from a joint distribution. The objective is to identify the subset that achieves the highest expected reward with high probability using as few queries as possible. We present distribution-dependent lower bounds based on a particular construction which force a learner to consider all N-choose-K subsets, and match naive extensions of known upper bounds in the bandit setting obtained by treating each subset as a separate arm. Nevertheless, we present evidence that exhaustive search may be avoided for certain, favorable distributions because the influence of high-order order correlations may be dominated by lower order statistics. Finally, we present an algorithm and analysis for independent arms, which mitigates the surprising non-trivial information occlusion that occurs due to only observing the max in the subset. This may inform strategies for more general dependent measures, and we complement these result with independent-arm lower bounds.

\end{abstract}

\section{Introduction}
This paper addresses a variant of the stochastic multi-armed bandit problem, where given $n$ arms associated with random variables $X_1,\dots,X_n$, and some fixed $1 \leq k \leq n$, the goal is to identify the subset $S \in \binom{[n]}{k}$ that maximizes the objective $\E\left[ \max_{i \in S} X_i \right]$. We refer to this problem as ``Best-of-K'' bandits to reflect the reward structure and the limited information setting where, at each round, a player queries a set $S$ of size at most $k$, and only receives information about arms $X_i : i \in S$: e.g. the vector of values of all arms in $S$, $\{ X_i: i \in S\} $ (semi-bandit), the index of a maximizer (marked bandit), or just the maximum reward over all arms  $\max_{i \in S} X_i$ (bandit).  The game and its valid forms of feedback are formally defined in Figure~\ref{fig:bandits_game}.

While approximating the Best-of-K problem and its generalizations have been given considerable attention from a computational angle, in the regret setting~\citep{yue2011linear, Hofmann:2011:PMI:2063576.2063618, Raman:2012:OLD:2339530.2339642, radlinski2008learning, yue2011linear, streeter2009online}, this work aims at characterizing its intrinsic statistical difficulty as an identification problem. Not only do identification algorithms typically imply low regret algorithms by first exploring and then exploiting, every result in this paper can be easily extended to the PAC learning setting where we aim to find a set whose reward is within $\epsilon$ of the optimal, a pure-exploration setting of interest for science applications \citep{kaufmann2014complexity,kaufmann2013information,hao2013limited}. 

For joint reward distributions with high-order correlations, we present distribution-dependent lower bounds which force a learner to consider all subsets $S \in \binom{[n]}{k}$ in each feedback model of interest, and match naive extensions of known upper bounds in the bandit setting obtained by treating each subset $S$ as a separate arm. Nevertheless, we present evidence that exhaustive search may be avoided for certain, favorable distributions because the influence of high-order order correlations may be dominated by lower order statistics. Finally, we present an algorithm and analysis for independent arms, which mitigates the surprising non-trivial information occlusion that occurs in the bandit and marked bandit feedback models. This may inform strategies for more general dependent measures, and we complement these result with independent-arm lower bounds. 

\subsection{Motivation}

In the setting where $X_i \in \{0,1\}$, one can interpret the objective $\max_{S \in \binom{[n]}{k}}\max_{i \in S} X_i$ as trying to find the set of items which affords the greatest coverage. For example, instead of using spread spectrum antibiotics which have come under fire for leading to drug-resistant ``super bugs'' \citep{SuperBug}, consider the doctor that desires to identify the best $k$ subset of narrow spectrum antibiotics that leads to as many favorable outcomes as possible. Here each draw from $X_i$ represents the $i$th treatment working on a random patient, and for antibiotics, we may assume that there are no synergistic effects between different drugs in the treatment. Thus, the antibiotics example falls under the {\em bandit feedback} setting since $k$ treatments are selected but it is only observed if at least one $k$-tuple of treatment led to a favorable outcome: no information is observed about any particular treatment. 

Now consider content recommendation tasks where $k$ items are suggested and the user clicks on either 1 or none. Here each draw from $X_i$ represents a user's potential interest in the $i$-th item, which we assume is independent of the other items which are shown with it. Nevertheless, due to the variety and complexity of users' preferences, the $X_i$'s have a highly dependent joint distribution, and we only get to observe {\em marked-bandit} feedback, namely one item which the user has clicked on. Our final example comes from virology where multiple experiments are prepared and performed $k$ at a time, resulting in $k$ simultaneous, noisy responses \citep{hao2013limited}; this motivates our consideration of the {\em semi-bandit} feedback setting.

\begin{figure}[t]
\centering
\fbox{\parbox{1.0\textwidth}{ 
\textbf{\underline{Best-of-$k$ Bandits Game}}\\[2pt]
for $t=1,2,...$\\[6pt]
\indent \hspace{.4cm} Player picks $S_t \in \mathcal{S}$ and adversary simultaneously picks $x_t \in \{ 0,1 \}^n$\\[6pt]
\indent \hspace{.4cm} Player observes $ \begin{cases} 
 \text{Bandit feedback:} &\displaystyle\max_{i \in S_t} x_{t,i}  \\[6pt]
 \text{Marked-Bandit feedback: } & \begin{cases} \emptyset & \text{ if }  x_{t,i} = 0 \ ,\ \ \forall {i \in S_t} \\ 
 								      \text{unif}(\displaystyle\arg\max_{i \in S_t} x_{t,i})  & \text{ otherwise. }\end{cases}\\[6pt] 
 \text{Semi-bandit feedback: } &x_{t,i} \ \ \forall i \in S_t  \end{cases} $                  
}}
\caption{Best-of-$k$ Bandits game for the different types of feedback considered. While this work is primarily interested in stochastic adversaries, our lower bound construction also has consequences for non-stochastic adversaries. Moreover, in marked feedback, we might consider non-uniform and even adversarial marking. \label{fig:bandits_game}}
\end{figure}

\subsection{Problem Description }
We denote $[n] = \{1,2,\dots,n\}$. For a finite set $W$, we let $2^{W}$ denote its power set, $\binom{W}{p}$ denote the set of all subsets of $W$ of size $p$, and write $V \sim \Unif[W,p]$ to denote that $V$ is drawn uniformly from $\binom{W}{p}$. If $X$ is a length $n$ vector (binary, real or otherwise) and $W \subset [n]$, we let $X_W$ denote the sub-vector indexed by entries $i \in W$.

In what follows, let $X = (X_1,\dots,X_n)$ be a random vector drawn from the probability distribution $\nu$ over $\{0,1\}^n$. We refer to the index $i \in [n]$ as the $i$-th arm, and let $\nu_i$ denote the marginal distribution of its corresponding entry in $X$, e.g. $(\E_\nu[X] )_i = \E_{\nu_i} [ X_i ]$. We define $\calS := \binom{[n]}{k}$, and for a given $S \in \calS$, we we call $\Exp[\max_{i \in S} X_i]$ the \emph{expected reward} of $S$, and refer casually to the random instantiations $\max_{i \in S} X_i$ as simply the reward of $S$.

At each time $t$, nature draws a rewards vector $x_t = X$ where $X$ is i.i.d from $\nu$. Simultaneously, our algorithm \emph{queries} a subset of $S_t \in \calS$ of $k$ arms, and we refer to the entries $i \in S_t$ as the arms \emph{pulled} by the query. As we will describe later, this problem has previously been studied in a \emph{regret} framework, where a time horizon $T \in \mathbb{N}$ is fixed and an algorithm's objective is to minimize its regret 
\begin{align}
R_\nu(T) = T \max_{S \in \mathcal{S}} \E_\nu[  \max_{i \in S} X_i ]  - \E_\nu[ \sum_{t=1}^T \max_{i \in S_t} X_i ].
\end{align}
In this work, we are more concerned with the problem of identifying the best subset of $k$ arms. More precisely, for a given measure $\nu$, denote the optimal subset 
\begin{align}
S^* := \arg\max_{S \in \mathcal{S}} \E_\nu \left[ \max_{i \in \mathcal{S}} X_{i} \right]
\end{align}
and let $T_S$ denote the (possibly random) number of times a particular subset $S \in \binom{[n]}{k}$ has been played before our algorithm terminates. The identification problem is then
\begin{defn}[Best-of-K Subset Identification] For any measure $\nu$ and fixed $\delta \in (0,1)$, return an estimate $\widehat{S}$ such that $\P_\nu( \widehat{S} \neq S^* ) \leq \delta$, and which minimizes the sum $\sum_{S \in \binom{[n]}{k}} T_S $ either in expectation, or with high probability. 
\end{defn}
Again, we remind the reader that an algorithm for Best-of-K Subset Identification can be extended to active PAC learning algorithm, and to an online learning algorithm with low regret (with high probability) \citep{kaufmann2014complexity,kaufmann2013information,hao2013limited}.

\subsection{Related Work}
Variants of Best-of-K have been studied extensively in the context of online recommendation and ad placement \citep{yue2011linear, Hofmann:2011:PMI:2063576.2063618, Raman:2012:OLD:2339530.2339642}. For example, \cite{radlinski2008learning}  introduces ``Ranked Bandits" where the arms $X_i$ are stochastic random variables, which take a value $1$ if the $t$-th user finds item $i$ relevant, and $0$ otherwise. The goal is to recommend an ordered list of items $S = (i_1,\dots,i_k)$ which maximizes the probability of a click on any item in the list, i.e. $\max_{i \in S} X_i$, and observes the first item (if any) that the user clicked on. \cite{streeter2009online} generalizes to online maximization of a sequence of monotone, submodular function $\{F_t(S)\}_{1\le t \le T}$ subject to knap-sack constraints $|S| \le k$, under a variety of feedback models. Since the function $S \mapsto \max_{i \in S} X_i$ is submodular, identifying $S^*$ corresponds to special case of optimizing the monotone, submodular function $F(S) := \Exp[\max_{i \in S} X_i]$ subject to these same constraints.

 \cite{streeter2009online}, \cite{yue2011linear}, and \cite{radlinski2008learning} propose online variants of a well-known greedy offline submodular optimization algorithm (see, for example \cite{iyer2013submodular}) , which attain $(1 - \frac{1}{e})$ approximate regret guarantees of the form 
\begin{eqnarray}
\sum_{t=1}^T F_t(S_t) - \left(1- \frac{1}{e}\right)\max_{S^*: |S^*| \le k} F_t(S^*) \leq R(T)
\end{eqnarray}
where $R(T)$ is some regret term that decays as $O(\mathrm{poly}(n,k)) \cdot o( T)$. Computationally, this $1- \frac{1}{e}$ is the best one could hope: Best-of-K and Ranked Bandits are online variants of the Max-K-Coverage problem, which cannot be approximated to within a factor of $1- \frac{1}{e} + \epsilon$ for any fixed $\epsilon > 0$ under standard hardness assumptions~\citep{vazirani2013approximation}. For completeness, we provide a formal reduction from Best-of-K identification to Max-K-Coverage in Appendix~\ref{AppCov}.

\subsection{Our Contributions}

Focusing on the stochastic pure-exploration setting with binary rewards, our contributions are as follows:
\begin{itemize}
\item We propose a family of joint distributions such that any algorithm that solves the best of $k$ identification problem with high probability must essentially query all $\binom{n}{k}$ combinations of arms. Our lower bounds for the bandit case are nearly matched by trivial identification and regret algorithms that treat each $k$-subset as an independent arm. For semi-bandit feedback, our lower bounds are exponentially higher in $k$ than those for bandit feedback (though still requiring exhaustive search). To better understand this gap, we sketch an upper bound that achieves the lower bound for a particular instance of our construction. While in the general binary case, the difficulty of marked bandit feedback is sandwiched between bandit and semi-bandit feedback, in our particular construction we show that marked bandit feedback has no benefit over bandit feedback. In particular, for worst-case instances, our lower bounds for marked bandits are matched by upper bounds based on algorithms which \emph{only take advantage of bandit feedback}.
\item Our construction plants a $k$-wise dependent set $S^*$ among $\binom{n}{k}-1$  k-wise independent sets, creating a needle-in-a-haystack scenario. One weakness of this construction is that the gap between the rewards of the best and second best subset are exponentially small in $k$. This is particular to our construction, but not to our analysis: We present a partial converse which establishes that, for any two $k-1$-wise independent distributions defined over $\{0,1\}^k$ with identical marginal means $\mu$, the difference in expected reward is exponentially small in $k$\footnote{Note that our construction requires all subset of $k-1$ of $S^*$ to be independent}. This begs the question: can low order correlation statistics allows us to neglect higher order dependencies? And can this property be exploited to avoid combinatorially large sample complexity in favorable scenarios with moderate gaps?
\item We lay the groundwork for algorithms for identification under favorable, though still dependent, measures by designing a computationally efficient algorithm for {\em independent} measures for the marked, semi-bandit, and bandit feedback models. Though independent semi-bandits is straightforward \citep{batch_mab}, special care needs to be taken in order to address the information occlusion that occurs in the bandit and marked-bandit models, even in this simplified setting. We provide nearly matching lower bounds, and conclude that even for independent measures, bandit feedback may require exponentially (in $k$) more samples than in the semi-bandit setting.
\end{itemize}

\section{Lower Bound for Dependent Arms} 

Intuitively, the best-of-$k$ problem is hard for the dependent case because the high reward subsets may appear as a collection of individually low-pay off arms  if not sampled together. For instance, for $k=2$, if $X_1 = \text{Bernoulli}(1/2)$, $X_2 = 1-X_1$, and $X_i = \text{Bernoulli}(3/4)$ for all $3 \leq i \leq n$, then clearly $\E[ \max\{X_1,X_2\} ] = 1$ is the best subset because $\E[ \max\{ X_1,X_i \}] = 1-(1/2)(1/4) = 7/8$ and $\E[ \max\{ X_i,X_j\} ] = 1- (1/4)^2 = 15/16$ for all $3 \leq i \leq j \leq n$. However, identifying set $\{1,2\}$ appears difficult as presumably one would have to consider all $\binom{n}{2}$ sets since if $X_1$ and $X_2$ are not queried together, they appear as $\text{Binomial}(1/2)$.  

Our lower bound generalizes this construction by introducing a measure $\nu$ such that (1) the arms in the optimal set $S^*$ are dependent but (2) the arms in every other non-optimal subset of arms $S \in \calS - S^*$ are mutually independent. This construction amounts to hiding a ``needle-in-a-haystack'' $S^*$ among all other $\binom{n}{k}-1$ subsets, requiring any possibly identification to examine most elements of $\mathcal{S}$. 

We now state our theorem, which characterizes the difficulty of recovering $S^*$ arms in terms of the gap $\Delta$  between the expected reward of $S^*$ and of the second best subset
\begin{align}\label{GapEq}
\Delta :=  \E_\nu \left[ \max_{i \in S^*} X_{i} \right] - \max_{S \in \mathcal{S} \setminus S^*}  \E_\nu \left[ \max_{i \in S} X_{i} \right]
\end{align}

\begin{theorem}[Dependent] \label{lower_bandit_identification}
Fix $k,n \in \mathbb{N}$ such that $2 \leq k < n$. For any $\epsilon \in (0,1]$ and $\mu \in (0,1/2]$ there exists a distribution $\nu$ with $\Delta = \epsilon \mu^k$ such that any algorithm that identifies $S^*$ with probability at least $1-\delta$ requires, in expectation, at least 
\begin{align*}
&(i) \quad \frac{4(1 - \epsilon (\frac{\mu}{1-\mu})^k) }{3} \cdot \big(1-(1-\mu)^k \big) \left(1-\mu\right)^k\binom{n}{k} \Delta^{-2} \log( \tfrac{1}{2\delta} ) \quad \text{(marked-)bandit, or} \\
&(ii) \quad \frac{2}{3}   \mu^{2k} (1-\epsilon) \binom{n}{k} \Delta^{-2} \log( \tfrac{1}{2\delta} ) \quad \text{semi-bandit}
\end{align*}
observations. In particular, for any $0 < \xi \leq (2k)^{-k}$ there exists a distribution $\nu$ with $\Delta = \xi$ that requires just $ \frac{1}{3} \binom{n}{k} \Delta^{-2} \log( \tfrac{1}{2\delta} )$ (marked-)bandit observations. And  for any $0 < \xi \leq 2^{-k-1}$ there exists a distribution $\nu$ with $\Delta = \xi$ that requires just $ \frac{1}{3} 2^{-2k}\binom{n}{k} \Delta^{-2} \log( \tfrac{1}{2\delta} )$ semi-bandit observations.  
\end{theorem}

\begin{rem}
Marked-bandit feedback provides strictly less information than semi-bandit feedback but at least as much as bandit feedback. The above lower bound for marked-bandit feedback and the nearly matching upper bound for bandit feedback remarked on below suggests that marked-bandit feedback may provide no more information than bandit feedback. However, the lower bound holds for just a particular construction and in Section~\ref{sec:upper_bound_independent} we show that there exist instances in which marked-bandit feedback provides substantially more information than merely bandit feedback.  
\end{rem}

In the construction of the lower bound, $S^* = [k]$ and all other subsets behave like completely independent arms. Each individual arm has mean $\mu$, i.e. $\E_\nu[X_i]=\mu$ for all $i$, so each $S \neq S^*$ has a bandit reward of $\E_\nu[ \max_{i\in S} X_i ] = 1-(1-\mu)^k$. The scaling $(1 - (1-\mu)^k) (1-\mu)^k$ in the  number of bandit and marked-bandit observations corresponds to the variance of this reward and captures the property that the number of times a set needs to be sampled to accurately predict its reward is proportional to its variance. Since $\mu\leq 1/2$, we note that the term $1 - \epsilon(\frac{\mu}{1-\mu})^k$ is typically very close to $1$, unless $\mu$ is nearly $1/2$ \emph{and} $\epsilon$ is nearly $1$. 

While the lower bound construction makes it necessary to consider each subset $S \in \binom{[n]}{k}$ individually for all forms of feedback feedback, semi-bandit feedback presumably allows one to detect dependencies much faster than bandit or marked-bandit feedback, resulting in an exponentially smaller bound in $k$. Indeed, Remark~\ref{rem:semibandit_tightness} describes an algorithm that uses the parity of the observed rewards that nearly achieves the lower bound for semi-bandits for the constructed instance when $\mu=1/2$. However, the authors are unaware of more general matching upper bounds for the semi-bandit setting and consider this a possible future avenue of research.

\subsection{Comparison with Known Upper Bounds}

By treating each set $S \in \mathcal{S}$ as an independent arm, standard best-arm identification algorithms can be applied to identify $S^*$. The KL-based {\em LUCB} algorithm  from \cite{kaufmann2013information} requires $O(\Delta_i^2 ( 1- (1-\mu)^k ) (1-\mu)^k \binom{n}{k} \cdot k \log n)$ samples,  matching our bandit lower bound up to a a multiplicative factor of $k \log n$ (which is typically dwarfed by $\binom{n}{k}$). The {\em lil'UCB} algorithm of \cite{jamieson2014lil} avoids paying this multiplicative $k \log n$ factor, but at the cost of not adapting to the variance term $( 1- (1-\mu)^k ) (1-\mu)^k $. Perhaps a KL- or variance-adaptive extension of {\em lil-UCB} could attain the best of both worlds. 

From a regret perspective, the exact construction as used in the proof of Theorem~\ref{lower_bandit_identification} can be used in Theorem~17 of \cite{kaufmann2014complexity} to state a lower bound on the regret after $T=\sum_{S \in \mathcal{S}} T_s$ bandit observations. Specifically, if an algorithm obtain a stochastic regret $R_T(\nu) =o(T^\alpha)$ for all $\alpha \in (0,1]$, then for all $S \in \binom{[n]}{k} - S^*$, we have $\lim\inf_{T \rightarrow \infty} \frac{\E_\nu[T_S] }{\log(T)} \geq \frac{( 1- (1-\mu)^k ) (1-\mu)^k}{\Delta^2}$ where $\Delta$ is given in Theorem~\ref{lower_bandit_identification}. Alternatively, in an adversarial setting, the above construction with $\mu = 1/2$ also implies a lower bound of $\sqrt{ 2^{-O(k)}\binom{n}{k} T} = \sqrt{\binom{\Omega(n)}{k} T}$ for any algorithm over a time budget $T$. Both of these regret bounds are matched by upper bounds found in \cite{bubeck2012regret}.

\subsection{Do Complicated Dependencies Require Small Gaps?}
While Theorem~\ref{lower_bandit_identification} proves the existence a family of instances in which $\binom{n}{k} \Delta^{-2} \log(1/\delta)$ samples are necessary to identify the best $k$-subset, the possible gaps $\Delta$ are restricted to be no larger than $\min\{\mu^k,(1-\mu)^k\}$. 
It is natural to wonder if this is an artifact of our analysis, a fundamental limitation of $k-1$-wise independent sets, or a property of dependent sets that we can potentially exploit in algorithms. The following theorem suggests, but does not go as far as to prove, that if there are very high-order dependencies, then these dependencies cannot produce gaps substantially larger than the range described by Theorem~\ref{lower_bandit_identification}. More precisely, the next theorem characterizes the maximum gap for $(k-1)$-wise independent instances.   

\begin{thm}\label{LowerBoundConverse} Let $X = (X_1,\dots,X_k)$ be a random variable supported on $\{0,1\}^k$ with $k-1$-wise independent marginal distributions, such that $\Exp[X_i] = \mu \in [0,1]$ for all $i \in \{1,\dots,k\}$. Then there is a one-to-one correspondence between joint distributions over $X$ and probability assignments $\Pr(X_1 = \dots = X_k = 0)$. When $\mu < 1/2$, all such assignments lie in the range 
\begin{eqnarray}
(1-\mu)^k \left(1 - \left(\frac{\mu}{1-\mu}\right)^{k_{even}}\right) \le \Pr(X_1 = \dots = X_k = 0) \le (1-\mu)^k \left(1 + \left(\frac{\mu}{1-\mu}\right)^{k_{odd}}\right)
\end{eqnarray}
Here, $k_{odd}$ is the largest odd integer $\le k$, and $k_{even}$ the largest even integer $\le k$. Moreover, when $\mu \ge 1/2$, all such assignments lie in the range
\begin{eqnarray}
0 \le \Pr(X_1 = \dots = X_k = 0) \le (1-\mu)^{k-1} 
\end{eqnarray}
\end{thm}

Noting that $\E[\max_{i \in [k]} X_i] = 1 - \P(X_1=\dots = X_k = 0)$, Theorem~\ref{LowerBoundConverse} implies that the difference between the largest possible and smallest possible expected rewards for a set of $k$ arms where each arm has mean $\mu$ and the distribution is $k-1$-wise independent is no greater than $(1-\mu)^k$, a gap of the same order of the gaps used in our lower bounds above. This implies that, in the absence of low order correlations, very high order correlations can only have a limited effect on the expected rewards of sets.

If it were possible to make more precise statements about the degree to which high order dependencies can influence the reward of a subset, strategies could exploit this diminishing returns property to more efficiently search for subsets while also maintaining large-time horizon optimality. In particular, one could use such bounds to rule out sets that need to be considered based just on their performance using lower order dependency statistics. To be clear, such algorithms would not contradict our lower bounds, but they may perform much better than trivial approaches in favorable conditions.

\section{Best of K with Independent Arms} \label{sec:upper_bound_independent}

	While the dependent case is of considerable practical interest, the remainder of this paper investigates the best-of-$k$ problem where $\nu$ is assumed to be a product distribution of $n$ independent Bernoulli distributions. We show that even in this presumably much simpler setting, there remain highly nontrivial algorithm design challenges related to the information occlusion that occurs in the bandit and marked-bandit feedback settings. We present an algorithm and analysis which tries to mitigate information occlusion which we hope can inform strategies for favorable instances of dependent measures. 

	Under the independent Bernoulli assumption, each arm is associated with a mean $\mu_i \in [0,1)$ and the expected reward of playing any set $S \in \binom{[n]}{2}$ is equal to $1- \prod_{i\in S} (1-\mu_i)$ and hence best subset of $k$ arms is precisely the set of arms with the greatest $k$ means $\mu_{i}$.

\subsection{Results}

	Without loss of generality, suppose the means are ordered $\mu_1 \ge \dots \mu_k > \mu_{k+1} \ge \dots \mu_n$. Assuming $\mu_k \ne \mu_{k+1}$ ensures that the set of top $k$ means is unique, though our results could be easily extended to a PAC Learning setting with little effort.  Define the gaps and variances via 
	\begin{eqnarray}
		\Delta_i := \begin{cases}\mu_i - \mu_{k+1} & \text{if } i \le k\\
		\mu_k - \mu_{i} & \text{if } i > k \\
		\end{cases} 
		& \text{ and } & V_i := \mu_i(1-\mu_i)
		\end{eqnarray}
	For $\tau > 0$, introduce the transformation
	\begin{eqnarray}
		\mathcal{T}_{n,\delta}(\tau) :=  \tau \log \left(\frac{16n\log_2e }{\delta}\log\left(\frac{8n  \tau \log_2 e}{\delta}\right) \right) = \tilde{\Theta}\left(\tau \log\left(\frac{n}{\delta}\right)\right)
	\end{eqnarray}
	where $\tilde{\Theta}(\cdot)$ hides logarithmic factors of its argument. We present guarantees for the Stagewise Elimination of Algorithm~\ref{SuccElim} in our three feedback models of interest; the broad brush strokes of our analysis are addressed in Appendix~\ref{RoadMapSec}, and the details are fleshed in the Appendices~\ref{TauApp} and~\ref{GeneralAnalySec}. Our first result is holds for semi-bandits, which slightly improves upon the best known result for the $k$-batch setting \citep{batch_mab} by adapting to unknown variances:

	\begin{thm}[Semi Bandit]\label{SemiBanditFinalUpperTheorem}
	With probability $1-\delta$, Algorithm~\ref{SuccElim} with semi-bandit feedback returns the arms with the top $k$ means using no more than
	\begin{eqnarray}
	8 \mathcal{T}_{n,\delta}(\tau_{\sigma(1)}) + \frac{4 }{k}\sum_{i=k+1}^n \mathcal{T}_{n,\delta}(\tau_{\sigma(i)}) = \tilde{O}\left(\left(\tau_{\sigma(1)} + \frac{1}{k} \sum_{i=k+1}^n \tau_{\sigma(i)}\right)\log\left(\frac{n}{\delta}\right)\right)
	\end{eqnarray}
	queries where
	\begin{equation}
		\begin{aligned}
		\tau_i := \frac{56}{\Delta_i} + \frac{256 }{\Delta_i^2}\begin{cases}  \max\{V_i, \max_{j > k} V_j\} & i \le k \\
		\max\{V_i, \max_{j \le k} V_j\} & i > k 
		\end{cases} 
		\end{aligned}
	\end{equation}
	and  $\sigma$ is a permutation so that $\tau_{\sigma(1)} \ge \tau_{\sigma(2)} \ge \dots \tau_{\sigma(n)}$. 
	\end{thm}
	The above result also holds in the more general setting where the rewards have arbitrary distributions bounded in $[0,1]$ almost surely (where $V_i$ is just the variance of arm $i$.) 

	In the marked-bandit and bandit settings, our upper bounds incur a dependence on information-sharing terms $H^M$ (marked) and $H^B$ (bandit) which capture the extent to which the $\max$ operator occludes information about the rewards of arms in each query. 

	\begin{thm}[Marked Bandit]\label{MarkedBanditFinalUpperTheorem}
	Suppose we require each query to pull exactly $k$ arms. Then Algorithm~\ref{SuccElim} with marked bandit feedback returns the arms with the top $k$ means with probability at least $1-\delta$ using no more than
	\begin{eqnarray}
	16 \mathcal{T}_{n,\delta}\left(\frac{\tau^M_{\sigma(1)}}{H^{M}}\right) + \frac{8 }{k}\sum_{i=k+1}^n \mathcal{T}_{n,\delta}\left(\frac{\tau^M_{\sigma(i)}}{H^{M}}\right) = \tilde{O}\left(\frac{\log\left(n/\delta\right)}{H^{M}}\left(\tau^M_{\sigma(1)} + \frac{1}{k} \sum_{i=k+1}^n \tau^M_{\sigma(i)}\right)\right)
	\end{eqnarray}
	queries. Here, $\tau_i^M$ is given by
	\begin{equation}\label{TauIMarked}
		\begin{aligned}
		\tau_i^M := \frac{56}{\Delta_i} + \frac{256 }{\Delta_i^2}\begin{cases}  \mu_i & i \le k \\
		\mu_k & i > k 
		\end{cases} 
		\end{aligned}
	\end{equation}
	$\sigma$ is a permutation so that $\tau_{\sigma(1)} \ge \tau_{\sigma(2)} \ge \dots \tau_{\sigma(n)}$, and $H^M$ is an ``information sharing term'' given by
	\begin{eqnarray}
	H^{M} := \Exp_{X_1,\dots,X_{k-1}}\left[\dfrac{1}{1 + \sum_{\ell \in [k-1]} \I(X_{\ell} = 1)}\right]
	\end{eqnarray}
	If we can pull fewer than $k$ arms per round, then we can achieve 
	\begin{eqnarray}\label{FewerThanK}
	8\max_{i \in [k-1]} i  \mathcal{T}(\tau^M_{\sigma(i)}) + \frac{8 }{kH^{M}}\sum_{i=2}^n \mathcal{T}_{n,\delta}\left(\tau^M_{\sigma(i)}\right) = \tilde{O}\left(\left(\max_{i \in \{1,k-1\}} i\tau_{\sigma(1)}^M + \frac{1}{kH^{M}} \sum_{i=2}^n \tau_{\sigma(i)}^M\right)\log\left(\frac{n}{\delta}\right)\right)
	\end{eqnarray}
	\end{thm}
	We remark that as long as the means are at no more than  $1 - c$, $\tau_i \le \frac{1}{c}\tau_{i}^M$, and thus the two differ by a constant factor when the means are not too close to $1$ (this difference comes from loosing $(1-\mu)$ term in a Bernoulli variance in the marked case). Furthermore, note that $H^M \ge \frac{1}{k}$. Hence, when we are allowed to pull fewer than $k$ arms per round, Stagewise Elimination with marked-bandit feedback does no worse than a standard LUCB algorithms for stochastic best arm identification. 

	When the means are on the order of $1/k$, then $H^M = \Omega(1)$, and thus Stagewise Eliminations gives the same guarantees for marked bandits as for semi bandits. The reason is that, when the means are $O(1/k)$, we can expect each query $S$ to have only a constant number of arms $\ell \in S$ for which $X_{\ell} = 1$, and so not much information is being lost by observing only one of them. 

	Finally, we note that our guarantees depend crucially on the fact that the marking is uniform. We conjecture that adversarial marking is as challenging as the bandit setting, whose guarantees are as follows:

	\begin{thm}[Bandit]\label{BanditFinalUpperTheorem}
	Suppose we require each query to pull exactly $k$ arms, $n \ge 7k/2$, and $\forall i: \mu_i < 1$. Then Algorithm~\ref{SuccElim} with bandit feedback returns the arms with the top $k$ means with probability at least $1-\delta$ using no more than
	\begin{eqnarray}
	20 \mathcal{T}_{n,\delta}\left(\frac{\tau^B_{\sigma(1)}}{H^{B}}\right) + \frac{5 }{k}\sum_{i=k+1}^n \mathcal{T}_{n,\delta}\left(\frac{\tau^B_{\sigma(i)}}{H^{B}}\right) = \tilde{O}\left(\frac{\log\left(n/\delta\right)}{H^{B}}\left(\tau^B_{\sigma(1)} + \frac{1}{k} \sum_{i=k+1}^n \tau^B_{\sigma(i)}\right)\right)
	\end{eqnarray}
	queries where $H^B:=\prod_{\ell \in [k-1]} (1-\mu_\ell)$ is an ``information sharing term'',
	\begin{align*}
	\tau_{i}^B &\leq \frac{66}{\Delta_i } +   \frac{2560}{\Delta_i^2}\begin{cases} 2(1-\mu_{k+1})\mu_i + (1-\mu_{k+1})^2 (1- H^{B} ) & i \le k \\
	2(1-\mu_{i})\mu_{k+1} + (1-\mu_{i})^2 (1- H^{B}) & i > k \end{cases}
	\end{align*}
	 and $\sigma$ is a permutation so that $\tau^B_{\sigma(1)} \ge \tau^B_{\sigma(2)} \ge\dots \tau^B_{\sigma(n)}$.
	 \end{thm}

	 The condition that $\mu_i < 1$ ensures identifiability (see Remark~\ref{MuConditionRemark}). The condition $n \ge 7k/2$ is an artifact of using a  Balancing Set $B$ defined in Algorithm~\ref{Balance}; without $B$, our algorithm succeeds for all $n \ge k$, albeit with slightly looser guarantees (see Remark~\ref{NConditionRemark}).

	 \begin{rem}\label{PostBanditThmRemark}
	 Suppose the means are greater than $\alpha(k)/k$ where $\alpha(k) \ge C\log k$ and $C$ is a  constant; for example, think $\alpha(k) = k/2$. Then $H^B \le (1 - \frac{\alpha(k)}{k})^k = O(\exp(-\alpha(k))) \ll 1/k$. Hence, Successive Elimination requires on the order of $\frac{1}{k} \cdot \frac{1}{ H^B} = \frac{\exp(\Omega(\alpha(k))}{k}$ more queries to identify the top $k$-arms than the classic stochastic MAB setting where you get to pull $1$-arm at a time, \emph{despite the seeming advantage that the bandit setting lets you pull $k$ arms per query}. When $\alpha(k) \ge C \log k$, then $\frac{\exp(\Omega(\alpha(k)))}{k}$ is at least polynomially large in $k$, and when $\alpha = \Omega(k)$, is exponentially large in $k$ (e.g, $\alpha(k) = k/2$). 

	 On the other hand, when the means are all on the order of $\alpha/k$ for $\alpha = O(1)$, then $H^B = \Omega(1)$, but the term $1-H^{B}$ is  at least $\Omega(\alpha)$. For this case, our sample complexity looks like
	 \begin{eqnarray}
	 \tilde{O}( \frac{\log(n/\delta) }{k}\sum_i \frac{\alpha/k + \alpha}{\Delta_i^2} + \frac{1}{\Delta_i} ) = \tilde{O}( \log(n/\delta) \sum_i \frac{\alpha }{\Delta_i^2}) 
	 \end{eqnarray}
	 which matches, but does not out-perform, the standard $1$-arm-per-query MAB guarantees, with variance adaptation (e.g., Theorem~\ref{SemiBanditFinalUpperTheorem} with $k=1$, note that $\alpha$ captures the variance). Hence, when the means are all roughly on the same order, it's never worse to pull $1$ arm at a time and observe its reward, than to pull $k$ and observe their max. Once the means vary wildly, however, this is certainly not true; we direct the reader to Remark~\ref{WorstCaseHBRemark} for further discussion. 
	 \end{rem}

\subsection{Algorithm}

	At each stage $t \in \{0,1,2,\dots\}$, our algorithm maintains an accept set $A_t \subset [n]$ of arms which we are are confident lie in the top $k$, a reject set $R_t \subset [n]$ of arms which we are confident lie in the bottom $n-k$, and an undecided set $U_t$ containing arms for which we have not yet rendered a decision. The main obstacle is to obtain estimates of the relative performance of $i \in U_t$, since the bandit and marked bandit observation models occlude isolated information about any one given arm in a pull. The key observation is that, if we sample $S \sim \Unif[U_t,k]$, then for $i,j \in U_t$, the following differences have the same sign as $\mu_i - \mu_j$ (stated formally in Lemma~\ref{DistrChar}): 
	\begin{equation}
	\begin{aligned}
	&\Exp[\max_{\ell \in S} X_{\ell} =1 \big{|} i \in S] - \Exp[\max_{\ell \in S} X_{\ell} \big{|} j \in S] \quad \text{ (bandits) }\quad \quad \text{and}  \\
	&\Pr(\text{ observe } X_{i} =1 \big{|} i \in S) - \Pr(\text{ observe } X_j = 1 \big{|} j \in S) \text{ (marked/semi-bandits)}
	\end{aligned}
	\end{equation}
	This motivates a sampling strategy where we partition $U_t$ uniformly at random into subsets $S_1,S_2,\dots,S_p$ of size $k$, and query each $S_{q}$, $q \in \{1,\dots,p\}$.  We record all arms $\ell \in S_{q}$ for which $X_{\ell} = 1$ in the semi/marked-bandit settings (Algorithm~\ref{PlayAndRecord}, Line~\ref{PMarkSemiLine}), and, in the bandit setting, mark down all arms in $S_q$ if we observe $\max_{\ell \in S_q}X_{\ell} = 1$ - i.e, we observe a reward of 1 (Algorithm~\ref{PlayAndRecord}, Line~\ref{PMarkBanditLine}). This recording procedure is summarized in Algorithm~\ref{PlayAndRecord}:

	\begin{algorithm}[ht]
	 	\textbf{Input} $S,\Splus \subset [n]$, $Y \in \R^n$\\
	 	 \textbf{Play} $S \cup \Splus$ \\
 		\Indp Semi/Marked Bandit Setting: $Y_{\ell} \leftarrow 1$ for all $\ell \in S$ for which we observe $X_{\ell} = 1$ \label{PMarkSemiLine} \\
	 	 Bandit Bandit Setting: If $A$ returns a reward of $1$, $Y_{\ell} \leftarrow 1$ for all $\ell \in S$\label{PMarkBanditLine}  \\
	 	\Indm \textbf{Return} $Y$ \\
	 	\caption{PlayAndRecord$(S,\Splus,Y)$ \label{PlayAndRecord}}
	 \end{algorithm}

	Note that PlayAndRecord$[S,\Splus,Y]$ plays a the union of $S$ and $\Splus$, but only records entries of $Y$ whose indices lie in $S$. UniformPlay (Algorithm~\ref{UniformPlay}) outlines our sampling strategy. Each call to UniformPlay$[U,A,R,k^{(1)}]$ returns a vector $Y \in \R^{n}$, supported on entries $i \in U$, for which 
	\begin{eqnarray}\label{ExpY}
	\Exp[Y_i] =  \begin{cases} \Pr_{S, \Splus}(\text{ observe } X_{i} =1 \big{|} i \in S \cup \Splus) & \text{marked/semi-bandit} \\
	 \Pr_{S,\Splus}(\max_{\ell \in S \cup \Splus} X_{\ell} =1  \big{|} i \in S \cup \Splus) & \text{bandit}
	\end{cases}
	\end{eqnarray}
	where $S \sim \Unif[U,k^{(1)}]$ and $\Splus$ is empty unless $|U_t|<k$ or we are allowed to pull fewer than $k$ arms per query in which case elements of $\Splus$ are drawn from $A \cup R$ as outlined in Algorithm~\ref{UniformPlay}, Line~\ref{TopOffStart} otherwise.

	There are a couple nuances worth mentioning. When $|U| < k$, we cannot sample $k$ arms from the undecided set $U$; hence UniformPlay pulls only $k^{(1)}$ from $U$ per query. If we are forced to pull exactly $k$ arms per query, UniformPlay adds in a ``Top-Off'' set of an additional $k - k^{(1)}$ arms, from $R$ and $A$ (Lines~\ref{TopOffStart}-\ref{TopOffEnd}). Furthermore, observe that lines~\ref{RemainderStart}-\ref{RemainderEnd} in UniformPlay carefully handle divisibility issues so as to not ``double mark'' entries $i \in U$, thus ensuring the correctness of Equation~\ref{ExpY}. Finally, note that each call to UniformPlay makes exactly $\ceil{|U|/k^{(1)}}$ queries.

	\begin{algorithm}[ht]
		\textbf{Inputs}: $U$, $A$, $R$, sample size $k^{(1)}$ \\
		\hangindent=0.7cm  Uniformly at random, partition $U$ into $p:= \floor{|U|/k^{(1)}}$ sets $S^{(1)},\dots,S^{(p)}$ of size $k^{(1)}$ and place remainders in $S^{(0)}$ // thus $S^{(1)},\dots,S^{(p)} \sim \Unif[U,k^{(1)}]$, but not indep \label{Partition} \\
		 \textbf{If} Require $k$ Arms per Pull \textbf{and} $k^{(1)} < k$  // Construct Top-Off Set $S^+$ \label{TopOffStart} \\
		 \Indp $k^{(2)} \leftarrow k - k^{(1)}$ // $k^{(2)} = |S^+|$ \label{k2Line1}\\
		 $\Splus \leftarrow\Unif[R, \min\{|R|,k^{(2)}\}]$ //sample as many items from reject as possible \\
		 \textbf{If} $|R| < k^{(2)}$: // sample remaining items from accept \\
			\Indp $\Splus \leftarrow R \cup \Unif[A, k^{(2)} - |R|]$ \\
		\Indm \Indm \textbf{Else} // Top-Off set unnecessary\\
		\Indp $S^+ \leftarrow \emptyset$, $k^{(2)} \leftarrow 0$  \label{TopOffEnd}\\
		\Indm  Initalize rewards vector $Y \leftarrow \mathbf{0} \in \R^{n}$  \\
		 \textbf{For} $q = 1,\dots,p$ \\ 
			\Indp $Y \leftarrow \text{PlayAndRecord}[S^{(q)},\Splus,Y]$ // only mark $S^{(q)}$ \\
		\Indm \textbf{If} $|S^{(0)}| > 0$    // if remainder \label{RemainderStart}\\
			\Indp Draw $S^{(0,+)} \sim \text{Unif}[U - S^{(0)}, k^{(1)} - |S^{(0)}|]$ // thus $S^{(0)} \cup S^{(0,+)} \sim \Unif[U,k^{(1)}]$ \\
			$Y \leftarrow \text{PlayAndRecord}[S^{(0)},  S^{(0,+)} \cup \Splus,Y]$ // only mark $S^{(0)}$ to avoid duplicate marking \label{RemainderEnd}\\
		\Indm \textbf{Return} $Y$
		\caption{UniformPlay$(U,A,R,k^{(1)})$}\label{UniformPlay}
	\end{algorithm}

	We deploy the passive sampling in UniformPlay in a stagewise successive elimination procedure formalized in Algorithm~\ref{SuccElim}. At each round $t = \{1,2,\dots\}$, use a doubling sample size to $T(t) := 2^t$, and set the $k^{(1)}$ parameter for UniformPlay to be $\min \{|U_t|,k\}$ (line~\ref{Tkline}). Next, we construct the sets $(U_t',R_t')$ from which UniformPlay samples: in the marked and semi-bandit setting, these are just $(U_t,A_t,R_t)$ (Line~\ref{UprimeSemi}), while in the bandit setting, they are obtained by from Algorithm~\ref{Balance} which transfers a couple low mean arms from $R_t$ into $U_t'$ (Line~\ref{UprimeBandit}). This procedure ameliorates the effect of information occlusion for the bandit case.

	Line~\ref{ForLoop} through~\ref{AvgEnd} average together $T(t) := 2^t$ independent, and identically distributed samples from $\text{UniformPlay}[U_t',R_t',A_t,k^{(1)}]$  to produce unbiased estimates $\hat{\mu}_{i,t}$ of the quantity $\Exp[Y_i]$ defined in Equation~\ref{ExpY}. $\hat{\mu}_{i,t}$ are Binomial, so we apply an empirical Bernstein's inequality from~\cite{maurer2009empirical} to build tight $1-\delta$ confidence intervals \begin{eqnarray}\label{empiricalConfidence}
		\widehat{C}_{i,t} := \sqrt{\frac{2 \hat{V} \log(8nt^2/\delta)}{T(t)}} + \frac{8 \log (8 nt^2/\delta)}{3 (T(t)-1)} & \text{where } & \hat{V}_{i,t} := \frac{T(t)\hat{\mu}_{i,t}(1 - \hat{\mu}_{i,t})}{T(t)-1}
	\end{eqnarray}
	Note that $\hat{V}_{i,t}$ coincide with the canonical definition of \emph{sample} variance. The variance-dependence of our confidence intervals is crucial; see Remarks~\ref{VarianceConfIntervalRemarkMarked} and~\ref{VarianceConfIntervalRemarkBandit} for more details. For any $\ell \leq |U_t|$ let
	\begin{eqnarray}
	\max_{j \in U_t}^{\ell} =  \ell\text{-th largest element} 
	\end{eqnarray}
	As mentioned above, Lemma~\ref{DistrChar} ensures $\Exp[\hat{\mu}_{i,t}] > \Exp[\hat{\mu}_{j,t}]$ if and only if $\mu_i > \mu_j$. Thus, accepting an arm for $\hat{\mu}_{i,t}$ is in the top $k$.

	\begin{algorithm}[ht]
		\textbf{Input} $S_1 = [n]$, Batch Size $k$\\
		\textbf{While} $|A_t| < k$ // fewer than k arms accepted \\
		\Indp Sample Size $T(t) \leftarrow 2^t$, Rewards Vector $Y^{(t)} \leftarrow \mathbf{0} \in \R^n$, $k^{(1)} \leftarrow \min \{|U_t|,k\}$ \label{Tkline}\\
		\hangindent=0.7cm $(U_t',R_t') \leftarrow (U_t,R_t)$ \quad// Sampling Sets for UniformPlay, identical to $U_t$ and $R_t$ in marked/semi bandits \label{UprimeSemi}\\
		\textbf{If} Bandit Setting  \quad// Add low mean arms from $R_t$ to $U_t$ \label{UprimeBandit}\\
		\Indp  $(U'_t,R'_t) \leftarrow \text{Balance}(U_t,R_t)$   \\
		 \Indm \textbf{For} $s = 1,2,\dots, T(t)$ \label{ForLoop} \\
			\Indp $Y^{(t)} \leftarrow Y^{(t)} + \text{UniformPlay}[U'_t,R'_t,A_t,k^{(1)}]$ \quad// get fresh samples\label{UnifPlay} \\
		\Indm $\hat{\mu}_{i,t} \leftarrow \frac{1}{T(t)} \cdot Y^{(t)}$ \label{hatmu} \quad// normalize \label{AvgEnd}\\
		$k_t \leftarrow k - |A_t|$ \\
		$A_{t+1} \leftarrow A_t \cup \{i \in U_t : \hat{\mu}_{i,t} - \hat{C}_{i,t} > \max_{j \in U_t}^{k_t + 1} \hat{\mu}_{j,t} + \hat{C}_{j,t}\}$ \quad// Equation~\ref{empiricalConfidence} \\  
		$R_{t+1} \leftarrow R_t \cup \{i \in U_t : \hat{\mu}_{i,t} + \hat{C}_{i,t} < \max_{j \in U_t}^{k_t } \hat{\mu}_{j,t} - \hat{C}_{j,t}\}$ \\
		\textbf{If} $|R_t| = n-k$ //$n-k$ arms rejected\\
			\Indp $A_{t+1} \leftarrow A_{t+1}  \cup U_t$ \label{RtTerminate} \\ 
		\Indm $U_{t+1} \leftarrow U_t - \{A_{t+1} \cup R_{t+1}\}$\\
		$t \leftarrow t + 1$\\
		\caption{Stagewise Elimination$(S,k,\delta)$ \label{SuccElim}}
	\end{algorithm}
	The Balance Procedure is described in Algorithm~\ref{Balance}, and ensures that $U_t'$ contains sufficiently many arms that don't have very high (top $k+1$) means. The motivation for the procedure is somewhat subtle, and we defer its discussion to the analysis in Appendix~\ref{BanditAnalSec}, following Remark~\ref{VarianceConfIntervalRemarkBandit}:

	\begin{algorithm}[ht]
		\textbf{Input} $U,R$ \\
		$B \sim \Unif[R,\max\{0,\ceil{\frac{5k^{(1)}}{2} - |U| - \frac{1}{2}}\}]$ //Balancing Set\\
		$U' \leftarrow U \cup B$ , $R' \leftarrow R -  B$ // Transfer $B$ from $R$ to $U$ \\
		\textbf{Return} $(U',R')$
		\caption{Balance($U,R$)\label{Balance}}
	\end{algorithm}

\section{Lower bound for Independent Arms}
In the bandit and marked-bandit settings, the upper bounds of the previous section depended on ``information sharing'' terms that quantified the degree to which  other arms occlude the performance of a particular arm in a played set. Indeed, great care was taken in the design of the algorithm to minimize impact of this information sharing. The next theorem shows that the upper bounds of the previous section for bandit and semi-bandit feedback are nearly tight up to a similarly defined information sharing term. 

\begin{theorem}[Independent]\label{lower_bound_independent}
Fix $1 \leq p \leq k \leq n$. Let $\nu = \prod_{i=1}^n \nu_i$ be a product distribution where each $\nu_i$ is an independent Bernoulli with mean $\mu_i$. Assume $\mu_1 \geq \dots \geq \mu_k > \mu_{k+1} \geq \dots \geq \mu_n$ (the ordering is unknown to any algorithm). At each time the algorithm queries a set $S' \in \binom{[n]}{p}$ and observes $\E[ \max_{i \in S'} X_i ]$. Then any algorithm that identifies the top $k$ arms with probability at least $1-\delta$ requires, in expectation, at least 
\begin{align*}
\Big( \max_{j = 1,\dots,n} \tau_j + \frac{1}{p} \sum_{j=1}^n \tau_j \Big) \log(  \tfrac{1}{2\delta} ) 
\end{align*}
observations where
\begin{align*}
&(i) \quad\tau_j = \begin{cases} \frac{(1-\mu_j-\Delta_j)}{\Delta_j^2} \frac{1-h_j + \mu_j h_j}{h_j} & \text{if $j > k$} \\ 
 											     \frac{(1-\mu_j)}{\Delta_j^2} \frac{1-h_j + (\mu_j-\Delta_j)h_j}{h_j} & \text{if $j \leq k$}
									 \end{cases} \quad \text{for bandit observations, and} \\
&(ii) \quad  \tau_j =  \begin{cases} \frac{ (1-\mu_j-\Delta_j) \mu_j}{\Delta_j^2} & \text{if $j > k$} \\ 
 											     \frac{(1-\mu_j) (\mu_j-\Delta_j) }{\Delta_j^2}& \text{if $j \leq k$}
									 \end{cases} \quad \text{for semi-bandit observations.}
\end{align*}
where $h_j = \max_{S \in \binom{[n]-j}{p-1}} \prod_{i \in a \setminus j} (1-\mu_i)$.
\end{theorem}

Our lower bounds apply to our upper bounds when $p=k$. In the bandit setting, considering $p<k$ reveals a trade-off between the information sharing term, which decreases with larger $p$, with the benefit of a $\frac{1}{p}$ factor gained from querying $p$ arms at once. One can construct different instances that are optimized by the entire range of $1\leq p \leq k$. Future research may consider varying the subset size in an adaptive setting to optimize this trade off. 

The information sharing terms defined in the upper and lower bounds correspond to the most pessimistic and optimistic scenarios, respectively, and result from applying coarse bounds in exchange for simpler proofs. Thus, our algorithm may fare considerably better in practice than is predicted by the upper bounds. Moreover, when $\max_{i} \mu_i - \min_i \mu_i$ is dominated by $\min_i \mu_i$ our upper and lower bounds differ by constant factors. 

Finally, we note that our upper and lower bounds for independent measures are tailored to Bernoulli payoffs, where the best $k$-subset corresponds to the top $k$ means. However, for general product distributions $\nu$ on $[0,1]^n$, this is no longer true (see Remark~\ref{NonBernoulliRemark}). This leaves open the question: how difficult is Best-of-K for general, independent bounded product measures? And, in the marked feedback setting (where one receives an index of the best element in the query), is this problem even well-posed? 
\section*{Acknowledgements}
We thank Elad Hazan for illuminating discussions regarding the computational complexity of the Best-of-K problem, and for pointing us to resources adressing online submodularity and approximate regret. Max Simchowitz is supported by an NSF GRFP award. Ben Recht and Kevin Jamieson are generously supported by ONR awards , N00014-15-1-2620, and  N00014-13-1-0129. BR is additionally generously supported by ONR award N00014-14-1-0024 and NSF awards CCF-1148243 and CCF-1217058.  This research is supported in part by gifts from Amazon Web Services, Google, IBM, SAP, The Thomas and Stacey Siebel Foundation, Adatao, Adobe, Apple Inc., Blue Goji, Bosch, Cisco, Cray, Cloudera, Ericsson, Facebook, Fujitsu, Guavus, HP, Huawei, Intel, Microsoft, Pivotal, Samsung, Schlumberger, Splunk, State Farm, Virdata and VMware.

\clearpage
\bibliographystyle{unsrt}
\bibliography{bestofkbandits}
\clearpage

\appendix
\section{Reduction from Max-K-Coverage to Best-of-K\label{AppCov}}

As in the main text, let $X = (X_1,\dots,X_n) \in \{0,1\}^n$ be a binary reward vector, let $\mathcal{S}^* = \{\arg\max_{S \in \binom{[n]}{k}}\Exp[\max_{i \in S} X_i]\}$ be set of all optimal $k$-subsets of $[n]$ (we allow for non-uniqueness), and define the gap $\Delta :=  \E_\nu \left[ \max_{i \in S^*} X_{i} \right] - \max_{S \in \mathcal{S} \setminus \mathcal{S}^*}  \E_\nu \left[ \max_{i \in S} X_{i} \right]$ as the minimum gap between the rewards of an optimal and sub-optimal $k$-set. We say $\tilde{S}$ is $\alpha-$optimal for $\alpha \le 1$ if $\Exp[\max_{i \in \tilde{S}} X_i ] \ge \alpha \Exp[\max_{i \in S^*} X_i]$, where $S^* \in \mathcal{S}^*$. We formally introduce the classical Max-K-Coverage problem:

\begin{defn}[Max-K-Coverage$(m,k,\mathcal{V})$] A Max-K-Coverage instance is a tuple $(m,k,\mathcal{V})$, where $\mathcal{V}$ is a collection of subsets $ V_1,\dots,V_n \in 2^{[m]}$. We say $S \subset \mathcal{V}$ is a solution to Max-K-Coverage if $|S| = k$ and $S$ maximizes $|\bigcup_{V_i \in S} V_i|$. Given $\alpha \le 1$, we say $S$ is an $\alpha$ approximation if $|\bigcup_{V_i \in S} V_i| \ge \alpha \max_{S' \in \binom{\mathcal{V}}{k}}|\bigcup_{V_i \in S'} V_i|$. 
\end{defn}
It is well known that Max-K-Coverage in NP-Hard, and cannot be approximated to within $\alpha = 1 - \frac{1}{e} + o(1)$ under standard hardness assumptions \cite{vazirani2013approximation}. The following theorem gives a reduction from Best of K Indentification (under any feedback model) to Max-K-Coverage:
\begin{thm}\label{Reduction}
Fix $\alpha \le 1$, and let $\mathcal{A}$ be an algorithm which indentifies an $\alpha$-optimal $k$-subset of $n$ arms probability in time polynomial in $n$, $k$, and $1/\Delta$, with probability at least $\eta$ (under any feedback model). Then there is a polynomial time $\alpha$-approximation algorithm for Max-K-Coverage$[m,k,\mathcal{V}]$which succeeds with probability at least $\eta$. When $\alpha = 1$, this implies a polynomial time algorithm for exact $\mathrm{Max-K-Coverage}[m,k,\mathcal{V}]$.
\end{thm}

\begin{proof}
Consider an instance of $\mathrm{Max-K-Coverage}[m,k,\mathcal{V}]$, and set $n = |\mathcal{V}|$. We construct a reward vector $X \in \{0,1\}^{n}$ as follows: At each time $t$, draw $\omega$ uniformly from $[m]$, and set $X_i := \I(\omega \in V_i)$. We run $\mathcal{A}$ on the reward vector $X$, and it returns a candidate set $\widehat{S} \in \binom{n}{[k]}$ which is $\alpha$-optimal with probability $\eta$. We then return the sets $V_i \in \mathcal{V}$ whose indicies lie in $\widehat{S}$. We show this reduction completes in polynomial time, and if $\widehat{S}$ is $\alpha$-optimal, then $\{V_i\}_{i \in \widehat{S}}$ is an $\alpha$-approximation for the Max-K-Coverage instance. 

\textbf{Correctness:} Since $\omega$ is uniform from $[m]$, the reward of a subset $S \subset [n]$ is $\Exp[\max_{i \in S}\I(\omega \in V_i)] = \Exp[\I( \omega \in \bigcup_{i \in S} V_i)] = \frac{|\bigcup_{i \in S} V_i|}{m} \propto |\bigcup_{i \in S} V_i|$. Hence, an $\alpha$-optimal subset $S$ corresponds to an $\alpha$-approximation to the Max-K-Coverage instance. 

\textbf{Runtime:} Let $R(n,k,\Delta) = O(\mathrm{poly}(n,k,1/\Delta))$ denote an upper bound runtime of $\mathcal{A}$, and let $T(n,k,\Delta) = O(\mathrm{poly}(n,k,1/\Delta))$ be an upper bound on the number of queries required by Algorithm $\mathcal{A}$ to return to $\alpha$-optimal $k$-subset. Note that sampling $\omega$ takes $O(m)$ time, and setting each $X_i(\omega)$ completes in time $O(mn)$. Moreover, the expected reward of any $S \in \binom{[n]}{k}$ lies in $\{0,\frac{1}{m},\dots,  1\}$, so $\Delta \le 1/m$. Thus, the runtime of our reduction is  $R(n,k,\Delta) + O(mn) \cdot T(n,k,\Delta)) \le R(n,k,1/m) + O(mn )  \cdot T(n,k,1/m)) = O(\mathrm{poly}(n,k,m))$.

\end{proof}

\begin{rem} Note that the parameter $m$ in the Max-K-Coverage instance shows up in the gap $\Delta$ in the runtime of the Max-K-Coverage instance. Our lower bound construction holds in the regime where $\Delta = \exp(-O(k))$, which morally corresponds to Max-K-Coverage instances in the regime where $m = \exp(\Omega(k))$.
\end{rem}

\section{High Level Analysis for Independent Upper Bound} \label{RoadMapSec}
	\subsection{Preliminaries}
		At each stage $t$ of Algorithm~\ref{SuccElim}, there are three sources of randomness we need to account for. First, there is the randomness over all events that occurred before we start sampling from UniformPlay: this randomness determines the undecided, accept, and rejected sets $U_t$, $A_t$, and $R_t$, as well as their modifications $U_t'$, and $R_t'$.  In what follows, we will define a so-called ``Data-Tuple'' $\mathcal{D}_t := (U_t,A_t,R_t,U_t',R_t')$ which represents the state of our algorithm, in round $t$, before collecting samples.

		The second source of randomness comes from the uniform partitioning of $U_t'$ into the sets $S^{(0)}, S^{(1)},\dots,S^{(q)}$ (Algorithm~\ref{UniformPlay}, Line~\ref{Partition}) and the draw of the Top-Off set $\Splus$ (Lines~\ref{TopOffStart}-\ref{TopOffStart}), at each call to UniformPlay. Finally, there is randomness over the values that the arms $X_{\ell} \in S \cup \Splus$ take, when pulled in PlayAndMark. To clear up any confusion, we define the probability and expectation operators 
		\begin{eqnarray}
		\Prt{\cdot} := \Pr[ \cdot \big{|} \mathcal{D}_t] & \text{and} & \Expt{\cdot} := \Exp[ \cdot \big{|} D_t]
		\end{eqnarray}
		$\Prt{\cdot}$ and $\Expt{\cdot}$ condition on the data in $\mathcal{D}_t$, and take expectations over the randomness in the partitioning of $U_t'$, draw of $\Splus$, and the values of each arm pulled. 

		Treating $\mathcal{D}_t$ as fixed, we will let $S$ denote a set with same distribution of one of the randomly partitioned subsets $S^{(1)},\dots,S^{(q)}$ of $U_t'$ in UniformPlay, $\Splus$ to denote a set with the distribution of the Top-Off set chosen in UniformPlay. Recall that the purpose of $\Splus$ is simply to ensure that we pull exact $k$ arms per query. If either $k^{(1)}  = k$, or we do not enforce exactly $k$-pulls per round, then $\Splus = \emptyset$. We remark that the distributions of $S$ and $\Splus$ are explicitly
		 \begin{eqnarray}
		S \sim \Unif[U_t',k^{(1)}] & \text{and} & \Splus \sim \begin{cases} \Unif[R_t',k^{(2)}] & |R_t'| \ge k^{(2)}\\
		R_t' \cup \Unif[A_t,k^{(2)} - |R_t'|]  & |R_t'| < k^{(2)}
		\end{cases}
		 \end{eqnarray}
		 Note that $\mathcal{D}_t$ exactly determines $k^{(1)} := |S|$, which we recall is defined at each round as $\min\{|U_t|,k\}$ (Algorithm~\ref{UniformPlay}, Line~\ref{Tkline}). It also determines the size of the Top-Off set $k^{(2)}$ (Algorithm~\ref{UniformPlay}, Lines~\ref{k2Line1} and~\ref{TopOffEnd}). We further note that the play $S^{(0)} \cup S^{(0,+)}$ (Algorithm~\ref{UniformPlay}, Lines~\ref{RemainderStart}-\ref{RemainderEnd} ) is also uniformly drawn as $\Unif[U_t',k^{(1)}]$, and hence has the same distribution of $S$. We also remark that
		 \begin{claim}
		 The sets $S$ and $\Splus$ are independent and disjoint under $\Pr_t$. In the marked and semi-bandit setting, there are always enough accepted/rejected arms in $|A_t \cup R_t|$ to ensure that we can fill $\Splus$ with $k^{(2)}$ arms.  In the bandit setting, there are sufficiently many accepted/rejected arms in $|A_t \cup R'_t|$ as long as $n \ge 7k/2$.
		 \end{claim}
		 This condition $n \ge 7k/2$ is an artifact of the balancing set in our algorithm, and is discussed in more detail in Section~\ref{BanditAnalSec}.

	\subsection{Guarantees for General Feedback Models\label{GeneralAnalySec}}

		The core of our analysis is common to the three feedback models. To handle bandits and marked/semi bandits settings simultaneous, we define a win function $\calW:  [n] \times 2^{[n]} \to \{0,1\}$ which reflects the recording strategy in PlayAndRecord
		\begin{eqnarray}
		\calW(i, S') = \begin{cases} 1 &\text{if bandit setting and } \max_{\ell \in S'}X_{\ell} = 1\\
		1 &\text{if marked/semi-bandit setting and observe } X_{i} = 1\\
		0 & \text{otherwise}
		\end{cases}
		\end{eqnarray}
		That is, $\text{PlayAndRecord}[S,\Splus,Y]$ sets $Y_i = 1 $ $\forall i \in S: \calW(i,S \cup \Splus) = 1$. The following lemma characterizes the distribution of our estimations $\hat{\mu}_{i,t} $
		\begin{lem}\label{DistrChar}
		\begin{eqnarray}
		\hat{\mu}_{i,t} \sim \frac{1}{T(t)}\mathrm{Binomial}(\bar{\mu}_{i,t}, T(t)) & \mathrm{and} & \Exp[\hat{V}_{i,t}] = V_{i,t}
		\end{eqnarray}
		where \begin{eqnarray}
		\bar{\mu}_{i,t} = \Exp_t[\calW(i, S \cup \Splus) \big{|} i \in S] & \text{and} & V_{i,t} := \bar{\mu}_{i,t}(1-\bar{\mu}_{i,t})
		\end{eqnarray}
		Moreover, in semi-bandit and marked bandit settings, and if $\mu_1 \le 1$ in the bandit setting, then given $i,j \in S_t$, $\bar{\mu}_{i,t} > \bar{\mu}_{j,t}$ if and only if $\mu_{i} > \mu_{j}$.
		\end{lem}
		\begin{rem}\label{NonBernoulliRemark} In the partial feedback models, the property that $\bar{\mu}_{i,t} > \bar{\mu}_{j,t}$ if and only if $\mu_{i} > \mu_{j}$ is quite particular to independent Bernoulli observations. The case of dependent Bernoullis measures is adressed by Theorem~\ref{lower_bandit_identification}. For independent, non-Bernoulli distributions, consider the setting where $n = 3$, $k = 2$, and let $X_1,X_2,X_3$ be independent, where $X_1 \overset{d}{=} X_2 \overset{a.s.}{=} 2/3$, and $X_3 \sim \text{Bernoulli}(1/2)$. Then, $\Exp[\max(X_1,X_2)] = 2/3$, while $\Exp[\max(X_1,X_3)] = \Exp[\max(X_2,X_3)] = \frac{1}{2} + \frac{1}{3} = 5/6$. Hence, if $S \sim \Unif[\{1,2,3\},2]$, $\Exp[\max_{\ell \in S} X_{\ell} \big{|} 3 \in S] > \Exp[\max_{\ell \in S} X_{\ell} \big{|}  2 \in S] = \Exp[\max_{\ell \in S} X_{\ell} \big{|} 1 \in S]$.
		\end{rem}
		\noindent  The last preliminary is to define the stage-wise comparator arms $c_{i,t}$ for $i \in U_t$:
		\begin{eqnarray}
		c_{i,t} := \begin{cases}  \min\{j \in U_t: j > k\}  & i \le k \\
		\max\{j \in U_t: j \le k\} & i > k \\
		\end{cases}
		\end{eqnarray}
		Intuitively, the comparator arm is the arm we are mostly to falsely accept instead of $i$ when $i \le k$, and falsely reject instead of $i$ when $i > k$. 
		\begin{rem} As long as the accept set $A_t$ only consists of arms $i \le k$, and $R_t$ only consists of arms $i > k$, $c_{i,t}$ is guaranteed to exists. Indeed, fix $i \in U_t$, and suppose $c_{i,t}$ does not exist. If $i \le k$, then this would mean that $U_t$ doesn't contain any rejected arms, but since $A_t$ only contains accepted arms, all rejected arms are in $R_t'$, in which case Algorithm~\ref{SuccElim} will have already terminated (Line~\ref{RtTerminate}). A similar contradiction arises when $i > k$. 
		\end{rem}
		Finally, we define the stagewise effective gaps
		\begin{equation}
		\begin{aligned}
		\Delta_{i,t} := |\bar{\mu}_{i,t} - \bar{\mu}_{c_{i,t},t}|  
		\end{aligned}
		\end{equation}
		
		Observe that, conditioned on the data in $\mathcal{D}_t$, the means $\bar{\mu}_{i,t}$, gaps $\Delta_{i,t}$ and the variances $V_{i,t}$ are all deterministic quantities. We now have the following guarantee for Algorithm~\ref{SuccElim}, which holds for the bandit, marked-bandit, and semi-bandit regimes:
		\\
		\begin{lem}[General Performance Guarantee for Successive Elimination]\label{SuccElimGuar} In the bandit, marked-bandit, and semi-bandit settings, the following is true for all $t \in \{0,1,\dots \}$ simultaneously with probability $1-\delta$: Algorithm~\ref{SuccElim} never rejects $i$ if $i \le k$ and never accepts $i$ if $i > k$. Furthermore, if for a stage $t$ and arm $i \in U_t$, the number of sample $T(t) := 2^t$ satisfies
		\begin{eqnarray}
		T(t) \ge T_{n,\delta}(\tau_{i,t}) := \tau_{i,t} \log \left(\frac{24 n }{\delta}\log\left(\frac{12 n  \tau_{i,t} }{\delta}\right) \right) 
		\end{eqnarray}
		where 
		\begin{eqnarray}
		\tau_{i,t} := \frac{56}{\Delta_{i,t}} + \frac{256 \max\{V_{i,t},V_{i,c_{i,t}}\}}{\Delta_{i,t}^2} 
		\end{eqnarray}
		then by the end of stage $t$, $i$ is accepted if $i \le k$ and rejected if $i > k+1$.
		\end{lem}
		\begin{rem} The above theorem holds quite generally, and its proof abstracts out most details of best-of-k observation model. In fact, it only requires that (1) for each $i \in U_t$, $\hat{\mu}_{i,t} \sim \frac{1}{T(t)} \text{Binomial}(\bar{\mu}_{i,t}, T(t))$ and (2) $\bar{\mu}_{i,t} > \bar{\mu}_{j,t} \iff \mu_i > \mu_j$. In our three settings of interest, both conditions are ensured by Lemma~\ref{DistrChar}. It also holds in the semi-bandit setting when the arms have arbitrary distributions, as long as the rewards are bounded in $[0,1]$.
		\end{rem}

		The final lemma captures the fact that each call to UniformPlay often makes fewer than $|U_t|$ queries to pull each arm in $U_t$:

		\begin{lem}\label{EfficiencyLemma}
			Suppose that, at round $t$, each call of uniformly play queries no more than $\alpha |U_t|/k$ times when $|U_t| \ge k$, and no more than $\alpha$ samples when $|U_t| \le k$. Let $t^*_i$ be the first stage at which $i \notin U_t$. Then, Algorithm~\ref{SuccElim} makes no more than the following number of queries
			\begin{eqnarray}
			4\alpha T(t_{\sigma(1)}^*) + \frac{2\alpha }{k}\sum_{i=k+1}^n  T(t^*_{\sigma(i)})
			\end{eqnarray}
			where $\sigma$ is permutation chosen so that $t^*_{\sigma(1)} \ge t_{\sigma(2)}^* \ge \dots \ge t^*_{\sigma(n)}$, and $T(t) = 2^t$, as above.
		\end{lem}

		\begin{rem}\label{AlphaRemark} In the marked-bandit and semi-bandit settings, it is straightforward to verify that one can take $\alpha = 2$ in the above lemma. This is because Algorithm~\ref{SuccElim} always calls UniformPlay (Line~\ref{UnifPlay}) on $U_t' = U_t$ (Algorithm~\ref{UprimeSemi}). Then, UniformPlay (Algorithm~\ref{UnifPlay}) partitions $U_t$ into at most $\ceil{|U_t|/k^{(1)}}$ queries $\Splus_q$. Recall that $k^{(1)} = \min \{|U_t|,k\}$ (Algorithm~\ref{SuccElim}, Line~\ref{Tkline}) so that $\ceil{|U_t|/k^{(1)}} \le \ceil{|U_t|/k} \le 2 |U_t|$ when $|U_t| \ge k$, while $|U_t|/k^{(1)} = 1 \le 2|U_t|$ once $|U_t| < k$. Controlling bound on $\alpha$ is slightly more involved in the bandit setting, and is addressed in Claim~\ref{BanditBalanceClaim}. 
		\end{rem}

	\subsection{Specializing the Results}
		
		In the following sections, we again condition on the data  $\mathcal{D}_t := (U_t,A_t,R_t,U_t',R_t')$. We proceed to compute the stage-wise means $\bar{\mu}_{i,t}$, variances $V_{i,t}$, and time parameters $\tau_{i,t}$ in Lemma~\ref{SuccElimGuar}. As a warm up, let's handle the semi-bandit case:

	\subsubsection{Semi-Bandits\label{SemiAnalSec}}
		In Semi-Bandits, $\bar{\mu}_{i,t} = \mu_i$, and so 
		\begin{equation}
		\begin{aligned}\label{tauiEqSemi}
		\tau_{i,t} &= \tau_i = \frac{256\max\{V_i,V_{c_{i,t}}\} }{\Delta_i^2} + \frac{56}{\Delta_i} 
		\end{aligned}
		\end{equation}
		as in Theorem~\ref{SemiBanditFinalUpperTheorem}. Noting that $c_{i,t} > k$ if $i \le k$, while $c_{i,t} \le k$ if $i > k$, we can bound
		\begin{eqnarray}
		V_{c_{i,t}} \le \begin{cases}  \max_{j > k} V_j  & i \le k \\
		 \max_{j \le k} V_j & i > k 
		 \end{cases}
		\end{eqnarray}
		Plugging the above display into Equation~\ref{tauiEqSemi}, we see that $\tau_{i,t} \le \tau_i$, as defined in Theorem~\ref{SemiBanditFinalUpperTheorem}. Combining this observation with Lemmas~\ref{SuccElimGuar} and~\ref{EfficiencyLemma} and Remark~\ref{AlphaRemark} concludes the proof of Theorem~\ref{SemiBanditFinalUpperTheorem}. Note that we pick up an extra factor of two, since we might end up collected at most $2 \mathcal{T}_{n,\delta}(\tau_i)$ samples before either accepting, or rejected, an arm $i$.

	\subsubsection{Marked Bandit\label{MarkedAnalSec}}

		In marked bandits, the limited feedback induces an ``information-sharing'' phenomenon between entries in the same pull.  We can now define the information sharing term as:
		\begin{equation}
		\begin{aligned}
		H_{i,j,t}^{M} &= \Exp_t\left[\dfrac{1}{1 +  \sum_{\ell \in S \cup \Splus - \{i,j\}} \I( X_{\ell} = 1) } \big{|} i \in S \right] 
		\end{aligned}
		\end{equation}
		where again $\Splus$ has the distribution as $\Splus$ in Algorithm~\ref{UniformPlay}, and the operator $\Exp_t$ treats the data in $\mathcal{D}_t$ as deterministic. The following remark explains the intuition behind $H_{i,j,t}^{M}$. 

		\begin{rem} When we query a set $S \cup \Splus$, marked bandit feedback uniformly selects one arm in $\{\ell \in S \cup \Splus : X_{\ell} = 1\}$ if its non-empty and selects no arms otherwise. Hence, the probability of receiving the feedback that $X_i = 1$ given that $i \in S$ and $X_{i} = 1$ is
		\begin{eqnarray}
		\Exp_t\left[\dfrac{1}{1 +  \sum_{\ell \in S \cup \Splus - \{i\}} \I( X_{\ell} = 1) } \big{|} i \in S \right] 
		\end{eqnarray}
		The above display captures how often the observation $X_i = 1$ is ``suppressed'' by another arm in the pull. In contrast, $H_{i,j,t}^M$ is precisely the probability of receiving feedback that $X_i = 1$, given that $X_i = 1$ and $i \in S$, but under a slightly different observation model where arm $j$ is never marked, and instead we observe  a marking uniformly  from $\{\ell \in S \cup \Splus - \{j\} : X_{\ell} = 1\}$. Hence, we can think of $H_{i,j,t}^M$ as capturing how often arms other than $j$ prevent us from observing $X_{i} = 1$. Note that the smaller $H_{i,j,t}^M$, the more the information about $X_i$ is suppressed. 
		\end{rem}
		We also remark on the scaling of $H_{i,j,t}^M$: 
		\begin{rem}
		Given $i \in S$, $|S \cup \Splus - \{i,j\}| \le k-1$, and thus $H_{i,j,t}^{M} \ge 1/k$. When the means are all high, its likely that $\Omega(k)$ arms $\ell$ in a query will have $X_{\ell} = 1$, and so we should expect that $H_{i,j,t}^{M} = O(1/k)$. When the means are small, say $O(1/k)$, then $H_{i,j,t}^m$ can be as large as $\Omega(1)$. This is because if we observe that $X_i = 1$ from a query $S \cup \Splus$, then its very likely that $X_i = 1$ in only a constant fraction of them. Stated otherwise: if the means are small, then seeing just one arm uniformly for which $X_i = 1$ as about as informative as seeing all the values of all the arms at once. 
		\end{rem}
		With this definition in place, we have
		\begin{prop}\label{MarkedBandProp} 
		\begin{eqnarray}
		\bar{\mu}_{i,t} - \bar{\mu}_{j,t} = (\mu_i - \mu_j){H}_{i,j,t}^M & \text{and} & V_{i,t} \le \mu_i H_{i,j,t}^{M}
		\end{eqnarray}
		As a consequence, we have
		\begin{equation}
		\begin{aligned}
		\tau_{i,t} &\le \frac{1}{H_{i,c_{i,t},t}^M}\left( \frac{256 \max\{\mu_i,\mu_{c_{i,t}}\}}{\Delta_i^2}  + \frac{56}{\Delta_i}\right)
		&\le \frac{\tau_i^M}{H_{i,c_{i,t},t}^M}
		\end{aligned}
		\end{equation}
		where $\tau_i^M$ is as in Equation~\ref{TauIMarked}. 
		\end{prop}
		\begin{rem}\label{VarianceConfIntervalRemarkMarked}
		In the above proposition, the variance term $ V_{i,t}$ has a factor $H_{i,j,t}^{M}$, which cancels out one of the $H_{i,j,t}^M$ terms from the gap $\Delta_{i,t}^2$. If we did not take advantage of a variance-adaptive confidence interval, our sample complexity would have to pay a factor of $(H_{i,j,t}^M)^{-2}$ instead of just $(H_{i,j,t}^M)^{-1}$.
		\end{rem}
		It is straightforward to give a worst case lower bound on $H_{i,j,t}^M$: 
		\begin{eqnarray}\label{MarkedMutualInf}
		H_{i,j,t}^M \ge H^{M} := \Exp_{X_1,\dots,X_{k-1}}\left[\dfrac{1}{1 + \sum_{\ell \in [k-1]} \I(X_{\ell})}\right]
		\end{eqnarray}
		As in the semi-bandit case, we can prove the first part Theorem~\ref{MarkedBanditFinalUpperTheorem} by stringing together Lemmas~\ref{SuccElimGuar} and~\ref{EfficiencyLemma} and Remark~\ref{AlphaRemark}, using Proposition~\ref{MarkedBandProp} to control $\tau_{i,t}$, and Equation~\ref{MarkedMutualInf} to give a worst case bound on the information sharing term. The argument for improving the sample complexity when we can pull fewer than $k$ arms per query (Equation~\ref{FewerThanK} in Theorem~\ref{MarkedBanditFinalUpperTheorem}) is a bit more delicate, and is deferred to section~\ref{FewerThanKAppen}.

	\subsubsection{Bandit Setting\label{BanditAnalSec}}

		Fix $i,j \in U_t'$. When UniformPlay pulls both $i$ and $j$ in the same query, we receive no relative information about $X_i$ versus $X_j$. Moreover, when another arm $X_{\ell}$ for $\ell \in S \cup \Splus - \{i\}$ takes a value $1$ (now assuming $j \notin S \cup \Splus$), it masks all information about $X_i$. Hence the analogue of the information sharing term $H^{M}_{i,j,t}$ is the product $H^B_{i,j,t} \cdot \kappa_1 $, where
		\begin{equation}
		\begin{aligned}
		H^B_{i,j,t} &:= \Prt{ \{  X_{\ell} = 0  : \forall \ell \in S \cup \Splus -\{i\}\}  \big{|} i \in S, j \notin S}  \quad \text{and}
		\\ \kappa_1 &:= \Prt{ j \notin S \cup \Splus  \big{|} i \in S} = \Prt{ j \notin S  \big{|} i \in S}
		\end{aligned}
		\end{equation}
		We defer the interested reader to the proof of Lemma~\ref{MuCompLem} in the appendix, which transparently derives the dependence on $H^B_{i,j,t} \cdot \kappa_1 $. We also show that, due the uniformity of the distribution of $S$,  $\kappa_1$ does not depend on the particular indices $i$ and $j$. 

		\begin{rem}\label{VarianceConfIntervalRemarkBandit} As in the Marked Bandit setting, we use a variance-adaptive confidence interval to cancel out one factor of $\kappa_1 H^{B}_{i,j,t}$. This turns out to incur a dependence on a parameter $\kappa_2$ - defined precisely in Section~\ref{BanditsAppendix} - which roughly corresponds to the inverse of the fraction of arms in $U_t'$ whose means do not lie in the top $k+1$.
		\end{rem}

		The balancing set $B$ is chosen precisely to control $\kappa_1$ and $\kappa_2$ It ensures that arms $i,j \in U_t$ \emph{do not} co-occur in the same query with constant probability (thus bounding $\kappa_1$ below) and that each draw of $S \sim \Unif[U_t',k^{(1)}]$ contains a good fraction of small mean arms as well (thus bounding $\kappa_2$ above). The following claim makes this precise:

		\begin{claim}\label{BanditBalanceClaim} Let $\kappa_1 = \Prt{ j \in S \big{|} i \in S}$ and $\kappa_2$ be as in Section~\ref{BanditsAppendix}, Equation~\ref{kappa2eq}. Then choice of
		\begin{eqnarray}
		|B| = \max\{0,\ceil{\frac{5k^{(1)}}{2} - |U| - \frac{1}{2}}\}
		\end{eqnarray}
		be as in Algorithm~\ref{Balance} ensures that $\kappa_1 \ge 1/2$, $\kappa_2 \le 2$, and $|U'| \le \frac{5}{2}|U|$. Moreover, as long as $n \ge \frac{7k}{2}$, Algorithm~\ref{Balance} can always sample $B$ from the reject set $R$. 
		\end{claim}
		\begin{rem}[Conditions on $n$]\label{NConditionRemark} The condition $n \ge 7k/2$ ensures that the balancing set $B$ is large enough to bound both $\kappa_1$ and $\kappa_2$. If we omit the balancing set, our algorithm can then identify the top $k$ means for any $n \ge k$, albeit with worse sample complexity guarantees. 
		\end{rem}

		\begin{prop}[Characterization of the Gaps]\label{FinalBanditGapProp}
		For all $i$,  $\Delta_{i,t} \ge \Delta_i H^{B}_{i,c_{i,t},t}$ and 
		\begin{equation}
		\begin{aligned}
		\frac{\max\{V_{i,t},V_{c_{i,t}}\}}{\Delta_{i,t}^2} &\leq (1+2\kappa_2) \frac{\max\{ (1-\mu_i) \bar{\mu}_{i,t}, (1-\mu_{{c_{i,t}},t}) \bar{\mu}_{{c_{i,t}},t} \} }{\Delta_i \ \Delta_{i,t}} \\
		&\leq \frac{1 + 2\kappa_2}{\kappa_1 H^{B}_{i,c_{i,t},t}} \cdot \frac{1}{\Delta_i^2}\begin{cases} 2(1-\mu_{k+1})\mu_i + (1-\mu_{k+1})^2 (1- H^{B}_{i,c_{i,t},t} ) & i \le k \\
		2(1-\mu_{i})\mu_{k+1} + (1-\mu_{i})^2 (1- H^{B}_{i,c_{i,t},t}) & i > k \end{cases}
		\end{aligned}
		\end{equation}
		where $\kappa_1$ and $\kappa_2$ are as in Claim~\ref{BanditBalanceClaim}.
		\end{prop}
		\begin{rem} Again, the variance-adaptivity of our confidence interval reduces our dependence on information-sharing from $(H_{i,j,t}^B)^{-2}$ to $(H_{i,j,t}^B)^{-1}$. 
		\end{rem}

		Plugging in $\kappa_1$ and $\kappa_2$ as bounded by Claim~\ref{BanditBalanceClaim},
		\begin{align}
		\tau_{i,t} 
		&\leq \frac{56}{\Delta_i H^{B}_{i,c_{i,t},t}} + \frac{2560}{H^{B}_{i,c_{i,t},t}} \cdot \frac{1}{\Delta_i^2}\begin{cases} 2(1-\mu_{k+1})\mu_i + (1-\mu_{k+1})^2 (1- H^{B}_{i,c_{i,t},t} ) & i \le k \\
		2(1-\mu_{i})\mu_{k+1} + (1-\mu_{i})^2 (1- H^{B}_{i,c_{i,t},t}) & i > k \end{cases}
		\end{align}
		We can wrap up the proof by a straightforward lower bound on $H_{i,j,t}^B$: 
		\begin{eqnarray}\label{WorstCaseHb}
		H_{i,j,t}^B \ge H^{B} :=\prod_{\ell \in [k-1]} (1-\mu_\ell)
		\end{eqnarray}
		and by invoking Claim~\ref{BanditBalanceClaim} to apply Lemma~\ref{SuccElimGuar} with $\alpha = 5/2$ as long as $n \geq 7k/2$.

		\begin{rem}[Conditions on $\mu_i$]\label{MuConditionRemark} The condition $\mu_i < 1$ ensures identifiability, since the top $k$ arms would be indistinguishable from any subset of $k$ arms which contains a arm $i$ for which $\mu_i = 1$. More quantitatively, this condition ensures that the information sharing term is nonzero. 
		\end{rem}

		\begin{rem}[Looseness of Equation~\ref{WorstCaseHb}]\label{WorstCaseHBRemark} When all the means $\mu_1,\dots,\mu_n$ are roughly on the same order, the worst case bound on $H_{i,j,t}^B$ in Equation~\ref{WorstCaseHb} is tight up to constants. Then, as remarked~\ref{PostBanditThmRemark}, there is never an advantage to looking at $k$-arms at a time and receiving their $\max$ over testing each arm individually. On the other hand, if the means vary widely in their magnitude, then there may very well be an advantage to querying $k$ arms at a time. 

		For example, suppose there are $k$ high means $\mu_1,\dots,\mu_k \ge 1/2$, and the remaining $n-k$ means are order $1/k$, and $n \gg k^2$. Then, in the early rounds ($|U_t| \gg k^2$), a random pull of $S$ will contain at most a constant number of means from with top $k$ with constant probability, and so $H_{i,j,t}^B = \Omega((1 - O(1/k))^k) = \Omega(1)$. From Lemma~\ref{MuCompLem}, we see empirical means $\widehat{\mu}_{i,t}$ of the high meaned arms will be $\Omega(1)$ variance. Thus, for early stages $t$, $\tau_{i,t} = O(1/\Delta_{i}^2)$. That is, we neither pay the penalty for a small information sharing term that we pay when the means are uniformly high, nor pay a factor of $k$ in the variance which would occur when the means are small. However, we still get to test $k$ arms a time, and hence querying $k$ arms at a time is roughly $k$ times as effective as pulling $1$.
		\end{rem}

\section{Computing $\tau_{i,t}$ with (Marked-)Bandit Feedback\label{TauApp}}
	
	\subsection{Preliminaries}
			We need to describe the distribution of two random subsets related to $S$. Again, taking the data $\mathcal{D}_t$ as given, define the sets $S_{-i\vee j}$ and $S_{-i \wedge j}$ as follows
			\begin{eqnarray}
			\Siandj \sim \Unif[U_t' - \{i,j\}, k^{(1)} - 2 ] & \text{and} & \Siorj \sim \Unif[U_t' - \{i,j\}, k^{(1)} - 1 ]] 
			\end{eqnarray}
			$\Siandj$ (read: ``S minus i \emph{and} j'') has the same distribution as $S - \{i,j\}$ given that both $i$ \emph{and} $j$ are in $S$. Similarly, $\Siorj$ (read: ``S minus i \emph{or} j'') has the same distribution as $S - \{i,j\} $ given that either $i$ or $j$ are in $S$, but not both. Equivalently, it has the same distribution as $S - \{i\} \big{|} i \in S, j \notin S$, and symmetrically, as $S - \{j\} \big{|} j \in S, i \notin S$. We will also define the constant
			\begin{eqnarray}\label{kappa1Def}
			\kappa_1 := \Pr_t(j \notin S \big{|} i \in S) = 1 - \frac{k^{(1)} - 1}{|U_t'| - 1}
			\end{eqnarray}
			Note that the definition of $\kappa_1$ is independent of $i$ and $j$, is deterministic given the data $\mathcal{D}_t$, and is well defined since Algorithm~\ref{SuccElim} always ensures $|U_t'| > 1$\footnote{the undecided set, and its modification, always contain at least two elements}. 
	 
	\subsection{Marked Bandits}
		
		In marked bandits, $U_t = U_t'$. Recall the definition 
			\begin{equation}
			\begin{aligned}
			H_{i,j,t}^{M} &= \Expt{\dfrac{1}{1 +  \sum_{\ell \in S \cup \Splus - \{i,j\}} \I( X_{\ell} = 1) } \big{|} i \in S} 
			\end{aligned}
			\end{equation}
			By splitting up into the case when $j \notin S \big{|} i \in S$ and $j \in S \big{|} i \in S$, we can also express
			\begin{equation}
			\begin{aligned}
			H_{i,j,t}^{M} &= \kappa_1 \Expt{\dfrac{1}{1 +  \sum_{\ell \in \Siorj \cup \Splus} \I( X_{\ell} = 1) }} \\
		  	&+ (1-\kappa_1) \Expt{\dfrac{1}{1 +  \sum_{\ell \in \Siandj \cup \Splus} \I( X_{\ell} = 1) }} 
			\end{aligned}
			\end{equation}
			Note that $\Siorj$ is well defined except when $|U_t - \{i,j\}| = |U_t| - 2 < k^{(1)} - 1$. Since $|U_t| \ge k^{(1)}$, this issue only occurs if $|U_t| = k^{(1)} - 1$, and thus $\kappa^{(1)} = 0$. To make our notation more compact, we let $\cardW{S'} = \sum_{\ell \in S'} \I( X_{\ell} = 1)$ (think ``cardinality of winners''). In this notation, the above display takes the form:
			\begin{equation}
			\begin{aligned}
			H_{i,j,t}^{M} &= \kappa_1 \Expt{\left(1+\cardW{\Siorj \cup \Splus}\right)^{-1}} + (1-\kappa_1) \Expt{\left(1+\cardW{\Siandj \cup \Splus}\right)^{-1}}
			\end{aligned}
			\end{equation}

			\begin{proof}[Proof of Proposition~\ref{MarkedBandProp}]
				Our goal is to bound $\bar{\mu}_{i,t}-\bar{\mu}_{j,t}$.

				By the law of total probability and the definition of $\kappa_1$, we have
				\begin{equation}
				\begin{aligned}
				\bar{\mu}_{i,t} &= \mu_i\Expt{\left(1 +\cardW{S - \{i\}} \right)^{-1} \big{|} i \in S}\\
				&= \mu_i\Prt{j \notin S_t \big{|} i \in S} \Expt{\left(1 + \cardW{\Siorj \cup \Splus}\right)^{-1}}\\ 
				&+ \mu_i\Prt{j \in S \big{|} i \in S} \Expt{\left(1 + \cardW{\{j\} \cup \Splus \cup \Siandj}\right)^{-1}} \\
				&= \mu_i\kappa_1 \Expt{\left(1 + \cardW{\Splus \cup \Siorj}\right)^{-1}} + \mu_i(1-\kappa_1)\Expt{\left(1 + \cardW{\{j\} \cup \Splus \cup \Siandj}\right)^{-1}} 
				\end{aligned}
				\end{equation}
				By conditioning on the events when arm $j$ takes the values of $1$ or zero, respectively, we can decompose $\Exp[(1 + \cardW{\{j\} \cup \Siandj})^{-1} ]$ into
				\begin{eqnarray}
				\mu_j\Expt{(2 + \cardW{ \Splus \cup \Siandj})^{-1}} + (1-\mu_j)\Expt{(1 + \cardW{ \Splus \cup \Siandj})^{-1} }
				\end{eqnarray}
				Substituting into the previous display and rearranging yields 
				\begin{eqnarray*}
				\bar{\mu}_{i,t} &=& \mu_i H^M_{i,j,t} + \mu_i\mu_j(1-\kappa_1)\Expt{\left(2 + \cardW{ \Splus \cup \Siandj}\right)^{-1} - \left(1 + \cardW{ \Splus \cup \Siandj}\right)} 
				\end{eqnarray*}
				Hence, we conclude
				\begin{eqnarray}
				\bar{\mu}_{i,t} - \bar{\mu}_{j,t} = (\mu_i - \mu_j){H}^M_{i,j,t}
				\end{eqnarray}
				To control $V_{i,t}$, we have $1 - \mu_{i,t} \le 1$, and
				\begin{eqnarray*}
				\bar{\mu}_{i,t} &=&  \mu_i\kappa_1\Exp\left[\left(1 + \cardW{\Splus \cup \Siorj}\right)^{-1}\right] + \mu_i(1-\kappa_1)\Exp\left[\left(1 + \cardW{\{j\} \cup \Splus \cup \Siandj}\right)^{-1} \right] \\
				&\le&  \mu_i\kappa_1\Exp\left[\left(1 + \cardW{\Splus \cup \Siorj}\right)^{-1}\right] + \mu_i(1-\kappa_1)\Exp\left[\left(1 + \cardW{\Splus \cup \Siandj}\right)^{-1} \right]\\
				&=&  \mu_i H^M_{i,j,t}
				\end{eqnarray*}
		\end{proof}

	\subsubsection{Improved Complexity With Fewer than $k$ Pulls per Query\label{FewerThanKAppen}}

		In this section, we prove the second part of Theorem~\ref{MarkedBanditFinalUpperTheorem}, which describes the setting where we permit fewer than $k$ pulls per query.
		\begin{proof}[Proof of Second Part of Theorem~\ref{MarkedBanditFinalUpperTheorem}]
				We mirror the proof of Lemma~\ref{EfficiencyLemma} in Section~\ref{EfficiencyLemProof}, and adopt its notation where $t^*_i$ be the first stage at which $i \notin U_t$, let $t_0$ be the first stage for which $|U_t|<k$. The same argument from Lemma~\ref{EfficiencyLemma} show that 
				\begin{eqnarray}
				\frac{2\alpha}{k} \sum_{i=1}^n T(t^*_i) + \sum_{t > t_0} \I(|U_t| > 0) T(t)
				\end{eqnarray}
				If $t_{fin}$ is the last stage of the algorithm for which $|U_t|>0$, then the doubling nature of the sample size lets us bound 
				\begin{eqnarray}
				\sum_{t > t_0} \I(|U_t| > 0) T(t) \le 2 T(t_{fin})
				\end{eqnarray}
				and clearly $t_{fin} = \min \{t^*_i : i \in U_{t_{fin}}\}$. We now bound $\tau_{i,j,t_{fin}}^M $ for $ i \in U_{t_{fin}}$ and any $j \in U_{t_{fin}}$. Indeed, recall that
				\begin{equation}
				\begin{aligned}
				H_{i,j,t}^{M} &= \kappa_1 \Expt{\left(1+\cardW{\Siorj \cup \Splus}\right)^{-1}} + (1-\kappa_1) \Expt{\left(1+\cardW{\Siandj \cup \Splus}\right)^{-1}}
				\end{aligned}
				\end{equation}
				When we are allowed to pull fewer than $k$ arms at once, then the ``Top-Off Set'' $\Splus$ is empty (Algorithm~\ref{SuccElim}, Line~\ref{TopOffStart}), and so the above is bounded above by $\max\{|\Siorj|,|\Siandj|\} \le |U_t| - 1$. Thus, we can easily bound $H_{i,j,t}^{M} \ge \frac{1}{|U_t|}$. In particular, this bound holds when $j = c_{i,t}$. Hence,
				\begin{eqnarray}
				\tau_{i,c_{i,t},t_{fin}} = \frac{\tau_i}{H_{i,c_{i,t},t_{fin}}} \le |U_{t_{fin}}|\cdot\tau_i
				\end{eqnarray}
				Recalling that $\mathcal{T}_{n,\delta}(\tau)$ is monotone, and applying the easy to verify identity that
				\begin{eqnarray}
				\mathcal{T}_{n,\delta}(\tau \cdot k') \le 2k' \mathcal{T}_{n,\delta}(\tau)
				\end{eqnarray}
				for all $k' \le n$, we have that for all $i \in U_{t_{fin}}$ that
				\begin{eqnarray}
				T(t^*_{i}) &\le& 2 \mathcal{T}_{n,\delta}(\tau_{i,j,t_{fin}}) \le 2 \mathcal{T}_{n,\delta}(\tau_i|U_{t_{fin}}|) \le 4|U_{t_{fin}}| \mathcal{T}_{n,\delta}(\tau_i)
				\end{eqnarray}
				If $\sigma$ is a permutation such that $\tau_{\sigma(1)} \ge \tau_{\sigma(2)} \ge \dots \ge \tau_{\sigma(n)}$, then for $i \in  U_{t_{fin}}$, $\tau_i \le \tau_{\sigma(|U_{t_{fin}}|)}$. Hence, taking the worst case over $|U_{t_{fin}}|$, we have
				\begin{eqnarray}
				\sum_{t > t_0} \I(|U_t| > 0) T(t) \le 2 T(t_{fin}) \le  8 |U_{t_{fin}}| \mathcal{T}(\tau_{\sigma(|U_{t_{fin}}|)}) \le 8\max_{i \in [k-1]} i \mathcal{T}(\tau_{\sigma(i)})
				\end{eqnarray}
			\end{proof}

	\subsection{Bandits\label{BanditsAppendix}}

		In this section, we drop the dependence on $t$ from the sets $U_t,A_t,R_t,U'_t,R_t'$, and let $B$ be the ``balancing set'' from Algorithm~\ref{Balance}; thus, $U' = U \cup B$, $A' = A - B$, and $R' = R - B$. Let $\kappa_1 = 1 -  \frac{k^{(1)} - 1}{|U'| - 1}$ be as in Equation~\ref{kappa1Def}, and let 
		\begin{eqnarray}\label{kappa2eq}
		\kappa_2 :=  \frac{k^{(1)} - 1}{|U'| - 2k^{(1)}} 
		\end{eqnarray}
		Finally, introduce the loss function $\mathcal{L}: 2^{[n]} \to \{0,1\}$ by $\mathcal{L}(S') = \I(\forall \ell \in S': X_{\ell} = 0)$.  Note $\Exp[\mathcal{L}(\{ \ell \})] = 1 - \mu_{\ell}$, and if two sets $S',S'' \subset [n]$ are disjoint, then $\mathcal{L}(S'\cup S'') = \mathcal{L}(S')\cdot \mathcal{L}(S'')$. Moreover, if $S'$ and $S''$ are almost-surely disjoint, random subset of $[n]$ which are independent given the data in $\mathcal{D}_t$, then $\Exp_t{\mathcal{L}(S' \cup S'')} = \Exp_t{\calL(S')} \cdot \Exp_t{\calL(S'')}$. Hence, the information sharing term can be expressed as 
		\begin{eqnarray}
		H_{i,j,t}^B:= \Expt{\calL(\Siorj \cup \Splus)} = \Expt{\calL(\Siorj)}\cdot\Expt{\calL(\Splus)}
		\end{eqnarray}
		and note that this term is nonzero as long as all the means are less than $1$, since with nonzero probability, any query of a nonempty set has a nonzero probability of all its arms taking the value zero. The following lemma gives an expression of $(1-\bar{\mu}_{i,t})$ in terms of $\kappa_1$, $\mu_i$, $H_{i,j,t}^B$, and an error term:

		\begin{lem}[Computation of $\bar{\mu}_{i,t}$]\label{MuCompLem}
		For any $i \ne j \in U'$, we have that
		\begin{equation}
		\begin{aligned}
		1-\bar{\mu}_{i,t} &= (1-\mu_i)\kappa_1 H_{i,j,t}^B \cdot \left( 1 +  \left(1-\mu_j\right)\Err_{i,j,t} \right)
		\end{aligned}
		\end{equation}
		where the term
		\begin{eqnarray}
		\Err_{i,j,t} := \frac{1-\kappa_1}{\kappa_1} \cdot \frac{\Expt{\calL(\Siandj)}}{\Expt{\calL[\Siorj}} = \frac{k^{(1)} - 1}{|U'| - k^{(1)}} \cdot \frac{\Expt{\calL(\Siandj)}}{\Expt{\calL[\Siorj}}
		\end{eqnarray}
		is symmetric in $i$ and $j$.
		\end{lem}
		\begin{proof}
			Using the independence of the arms, we have
			\begin{eqnarray*}
			1-\bar{\mu}_{i,t} &=& \Expt{\calL(S \cup \Splus) \big{|} i \in S} = (1-\mu_i)\Expt{\calL(S - \{i\}) \big{|} i \in S}\Expt{\calL(\Splus)}
			\end{eqnarray*}
			For $i \ne j \in U'$, we have
			\begin{eqnarray*}
			\Expt{\calL(S - \{i\}) \big{|} i \in S} &=& \kappa_1 \Expt{\calL(S - \{i\}) \big{|} i \in S, j \notin S} + (1-\kappa_1)\Expt{\calL(S-\{i\}) \big{|} i \in S, j \in S} \\
			&=& \kappa_1 \Expt{\calL(\Siorj)} + (1-\kappa_1)\Exp[\calL( \{j\} \cup \Siandj] \\
			&=& \kappa_1 \Expt{\calL(\Siorj)} + (1-\kappa_1)(1-\mu_j)\Exp[\calL( \Siandj)]\\
			&=& \kappa_1 \Expt{\calL(\Siorj)} \left( 1  + (1-\mu_j)\frac{1-\kappa_1}{\kappa_1} \cdot \frac{\Expt{\calL( \Siandj)}}{ \Expt{\calL(\Siorj)}} \right) 
			\end{eqnarray*}
			The result now follows from plugging in the above display into the first one, and using the definition of $\kappa_1$.
		\end{proof}

		\noindent Since both $H_{i,j,t}^B$ and $\Err_{i,j,t}$ are symmetric in $i$ and $j$, we get an exact expression for the gaps.
		\begin{cor}[Bandit Gaps]\label{BanditGaps}
		\begin{eqnarray}
				\bar{\mu}_{i,t} -\bar{\mu}_{j,t} =  \kappa_1 H_{i,j,t}\cdot \left(\mu_i - \mu_j\right)
				\end{eqnarray}
				In particular, $\bar{\mu}_{i,t} > \bar{\mu}_{j,t}$ if and only if $\mu_i > \mu_j$, and 
				\begin{eqnarray}
				\Delta_{i,t} =  \kappa_1 H_{i,c_{i,t},t}\cdot \left|\mu_i - \mu_{c_{i,t}}\right|
				\end{eqnarray}
		\end{cor}

		To get an expression for $\tau_{i,t}$, as defined in Lemma~\ref{SuccElimGuar}, we need to get an expression for the ration of the variance to the gap-squared, $\frac{\max\{V_{i,t},V_{c_{i,t}}\}}{\Delta_{i,t}^2}$. We decompose $V_{i,t} = (1-\bar{\mu}_{i,t})\bar{\mu}_{i,t}$, and similarly for $c_{i,t}$, and begin by bounding $(1-\bar{\mu}_{i,t})/\Delta_{i,t}$ and $(1 - \bar{\mu}_{c_{i,t},t})/\Delta_{i,t}$:
		\begin{lem}\label{VarLem1}
		\begin{eqnarray}
		\frac{1 - \bar{\mu}_{i,t}}{\Delta_{i,t}} \le \frac{(1 + 2\kappa_2)(1 - \mu_i)}{\Delta_i} & \text{and} &  \frac{1 - \bar{\mu}_{c_{i,t},t}}{\Delta_{i,t}} \le \frac{(1 + 2\kappa_2)(1 - \mu_{c_{i,t}})}{\Delta_i} 
		\end{eqnarray}
		\end{lem}
		This result uses $1 - \bar{\mu}_{i,t}$ to kill off one factor of $\kappa_1 H_{i,j,t}^B$ from the stagewise gaps $\Delta_{i,t}$, so that our final expression $\tau_{i,t}$ depends on the inverse information sharing term, and not its square.  The proof of the above lemma is somewhat delicate, and we defer it to the end of this section. Next, we need an upper bound on $\bar{\mu}_{i,t}$. Clearly, we can upper bound this quantity by $1$, but this can be loose when the means are small, and so we introduce the following lemma
		\begin{lem} \label{VarLem2}
		\begin{eqnarray}
		\hspace{.5in}&&\hspace{-.5in}\frac{\max\{(1-\mu_i)\bar{\mu}_{i,t}, (1-\mu_{c_{i,t}})\bar{\mu}_{c_{i,t},t}\}}{\Delta_{i,t}} \\
		&\le& \frac{1}{\kappa_1 \Delta_i H^{B}_{i,c_{i,t},t}} \begin{cases} 2(1-\mu_{k+1})\mu_i + (1-\mu_{k+1})^2 (1- H^{B}_{i,c_{i,t},t} ) & i \le k \\
		2(1-\mu_{i})\mu_{k+1} + (1-\mu_{i})^2 (1- H^{B}_{i,c_{i,t},t}) & i > k
		\end{cases}
		\end{eqnarray}
		\end{lem}
		Combining Corollary~\ref{BanditGaps}, Lemma~\ref{VarLem1} and~\ref{VarLem2}, establishes Proposition~\ref{FinalBanditGapProp}

	\subsubsection{Proof of Lemma~\ref{VarLem2}}

			We start out with a simple upper bound on $\bar{\mu}_i$ and $\bar{\mu}_{c_{i,t}}$:
			\begin{lem}\label{VarLem3}
			\begin{eqnarray}
			\bar{\mu}_{i,t} &\le& \mu_i + \mu_{c_{i,t}} + (1-\mu_i)(1 - H^B_{i,c_{i,t},t})
			\end{eqnarray}
			and similarly when we swap $i$ and $c_{i,t}$
			\end{lem}
			\begin{proof}[Proof of Lemma~\ref{VarLem3}]
				Let $c = c_{i,t}$. For $S' \in 2^{[n]}$, define the ``win'' function $\calW(S'): 1 - \calL(S')$ which takes a value of $1$ if $\exists \ell \in S': X_{\ell} = 1$. By a union bound, $\Exp[\calW(S' \cup S'')] \le \Exp[\calW(S')] + \Exp[\calW(S'')]$, even when $S'$ and $S''$ are dependent. Hence,
				\begin{eqnarray}
				\bar{\mu}_{i,t} &=& \Expt{\calW(S \cup \tilde{S}) \big{|} i \in S}\\
				&=& \Expt{\I(X_{i} = 1) \calW(S \cup \tilde{S}) \big{|} i \in S} + \Expt{\I(X_{i} \ne 1) \calW(S \cup \tilde{S}) \big{|} i \in S}  \\
				&\le& \mu_i + (1-\mu_i)\Expt{\calW(S - \{i\} \cup \tilde{S}) \big{|} i \in S}  
				\end{eqnarray}
				Now, using the union bound property of $\calW$, we have
				\begin{eqnarray}
				\Expt{\calW(S - \{i\} \cup \tilde{S}) \big{|} i \in S} \le \mu_{c} + \Expt{\calW(S - \{i\} - \{c\} \cup \tilde{S}) \big{|} i \in S}
				\end{eqnarray}
				Finally, by decomposing into the cases when $c \in S$ and $c \notin S$, we 
				\begin{eqnarray}
				\Expt{\calW(S - \{i\} - \{c\} \cup \tilde{S}) \big{|} i \in S} = \kappa_1 \Exp_t{\calW(\Siorc)} + (1-\kappa_1)\Expt{\Siandc}
				\end{eqnarray}
				Observe that $\Siandc \sim \Unif[U',k^{(1)} - 2]$, whereas $\Siorc \sim \Unif[U',k^{(1)} - 1]$; consequently, playing $\Siorc$ has a greater chance of yielding a win than $\Siandc$. Thus, we can bound 
				\begin{eqnarray}
				\Expt{\calW(S - \{i\} - \{c\} \cup \tilde{S}) \big{|} i \in S} \le \Expt{\calW(\Siorj)} = 1 - H_{i,c,t}^B
				\end{eqnarray}
			\end{proof}

			Now, Lemma~\ref{VarLem2} follows from the following claim, together with the expression for the gap $\Delta_{i,t}$ from Corollary~\ref{BanditGaps}:
			\begin{claim}
			\begin{eqnarray}
			\frac{\max \{(1-\mu_i) (\mu_i + \mu_{c_{i,t}}) ,(1-\mu_{c_{i,t}})(\mu_i + \mu_{c_{i,t}})\}}{|\mu_i - \mu_{c_{i,t}}|} \le \frac{2}{\Delta_i} \cdot \begin{cases} 
			 (1-\mu_{k+1})\mu_i & i \le k \\
			(1-\mu_i)\mu_k  & i > k
			\end{cases}
			\end{eqnarray}
			and 
			\begin{eqnarray}
			\frac{\max \{(1-\mu_i)^2,(1-\mu_{c_{i,t}})^2\}}{|\mu_i - \mu_{c_{i,t}}|} \le \frac{1}{\Delta_i} \begin{cases} (1 - \mu_{k+1})^2 & i \le k \\
			(1 - \mu_{i})^2 & i > k 
			\end{cases}
			\end{eqnarray}
			\end{claim}
			\begin{proof} Suppose first that $i > k$, so that $(1-\mu_{c_{i,t}})(\mu_i + \mu_{c_{i,t}}) \le (1-\mu_i) (\mu_i + \mu_{c_{i,t}}) \le  2(1-\mu_i)\mu_{c_{i,t}}$. Then,
			\begin{eqnarray}
			\frac{2(1-\mu_i)\mu_{c_{i,t}}}{|\mu_{c_{i,t}} - \mu_i|} &=& \frac{2(1-\mu_i)\mu_{c_{i,t}}}{\mu_{c_{i,t}} - \mu_i} \\
			&=& \frac{2(1-\mu_i)}{1 - \mu_i/\mu_{c_{i,t}}} \\
			&\le& \frac{2(1-\mu_i)}{1 - \mu_i/\mu_{k}} \\
			&\le& \frac{2(1-\mu_i)\mu_k}{\mu_k - \mu_i} \\
			&\le& \frac{2(1-\mu_i)\mu_k}{\Delta_i}
			\end{eqnarray}
			The rest follows from similar arguments.
			\end{proof}

	\subsubsection{Proof of Lemma~\ref{VarLem1}}

		Lemma~\ref{VarLem1} follows from the expression for the gaps in Corollary~\ref{BanditGaps}, and the following technical lemma:
		\begin{lem} 
		Fix $i \in U'$, and let $c \in U' \cap [k]$ if $i > k$ and $c \in U' - [k]$. Then, 
		\begin{eqnarray}
		\frac{1 - \bar{\mu}_{i,t}}{|\mu_i - \mu_c|} \le \frac{1 - \mu_i}{\Delta_i} \cdot \kappa_1(1 + 2\kappa_2)H_{i,j,c}
		\end{eqnarray}
		and
		\begin{eqnarray}
		\frac{1 - \bar{\mu}_{c,t}}{|\mu_i - \mu_c|} \le \frac{1 - \mu_c}{\Delta_i} \cdot \kappa_1(1 + 2\kappa_2)H_{i,j,c}
		\end{eqnarray}
		\end{lem}

		\begin{proof}
			By Lemma~\ref{MuCompLem},
			\begin{eqnarray}
			1 - \bar{\mu}_{i,t} &=&  (1-\mu_i)\kappa_1 H_{i,c,t} \left( 1 +  (1-\mu_c)\Err_{i,c,t} \right). \label{c3_0}
			\end{eqnarray}
			The following lemma, proved later, controls the term on $\Err_{i,c,t}$.
			\begin{lem}\label{kappaClaim}
				Suppose that $j \in [k+1]$ , and that the balancing set $B$ satisfies $B \cap [k] = \emptyset$. Then, for any $i\ne c \in U$ (where possibly $j \ne c$), we have
				\begin{eqnarray}
				(1-\mu_j)\Err_{i,c,t} &\le& \kappa_2.
				\end{eqnarray}
			\end{lem}

				When $i > k$,  $c \in [k]$ and $1-\bar{\mu}_{c,t} \le 1 - \bar{\mu}_{i,t}$ so that
				\begin{eqnarray}
				 \frac{1- \bar{\mu}_{c,t}}{|\bar{\mu}_{i,t} - \bar{\mu}_{c,t}|} &\le& \frac{1- \bar{\mu}_{i,t}}{|\bar{\mu}_{i,t} - \bar{\mu}_{c,t}|}\\
				  &\le& \frac{(1-\mu_i)\kappa_1 H_{i,c,t} (1+\kappa_2)}{|\bar{\mu}_{i,t} - \bar{\mu}_{c,t}|} \label{c3_2}\\
				  &\le& \frac{(1-\mu_i) (1+\kappa_2)}{|\mu_i - \mu_c|} \label{c3_3} \\
				  &\le& \frac{(1-\mu_i) (1+\kappa_2)}{\Delta_i} \label{c3_4}
				\end{eqnarray}
				where \eqref{c3_2} follows from combining \eqref{c3_0} and Lemma~\ref{kappaClaim}, \eqref{c3_3} follows from Corollary~\ref{BanditGaps}, and \eqref{c3_4} holds by $|\mu_i - \mu_c| \ge \max\{\Delta_i,\Delta_c\}$. 
				Moreover, swapping the roles of $c$ and $i$, we have that when $i \le k$,
				\begin{eqnarray}
				\frac{1- \bar{\mu}_{c,t}}{|\bar{\mu}_{i,t} - \bar{\mu}_{c,t}|} &\le& \frac{(1-\mu_c)\kappa_1 H_{i,c,t} (1+\kappa_2)}{|\bar{\mu}_{i,t} - \bar{\mu}_{c,t}|} \\
				&\le& \frac{(1-\mu_c) (1+\kappa_2)}{\Delta_i}.
				\end{eqnarray}
				The final case we need to deal with is the computation of $\frac{1- \bar{\mu}_{i,t}}{|\bar{\mu}_{i,t} - \bar{\mu}_{c,t}|}$ when $i \le k$. The problem is that it might be the case that $c > k+1$, impeding the application of Lemma~\ref{kappaClaim}. We get around this issue by breaking up into cases:

				(1) If $1 - \mu_{c}$ and $1 - \mu_i$ are on the same order, we are not in so much trouble. Indeed, if $1-\mu_{c} \le 2(1-\mu_i)$, then, we have
				\begin{eqnarray*}
				1- \bar{\mu}_{i,t}& =&(1-\mu_i)H_{i,c,t}\left(1 + (1-\mu_c)\Err_{i,c,t} \right) \\
				 &\le& (1-\mu_i)H_{i,c,t}\left(1 + 2(1- \mu_{k+1})\Err_{i,c,t}\right) \\
				&\le& (1-\mu_i)H_{i,c,t}\left(1 + 2\kappa_2  \right) 
				\end{eqnarray*}
				where the last step follows from  applying Lemma~\ref{kappaClaim} with $j = k+1$.

				(2) What happens when $1- \mu_{c} > 2(1- \mu_{i})$? Then we have
				\begin{eqnarray*}
				(\mu_i - \mu_{c})^{-1} (1-\mu_{c}) &=& \frac{1- \mu_{c}}{\Delta_i} \cdot \frac{\Delta_i}{\mu_i - \mu_{c}}  \\
				&=& \frac{1 - \mu_{k+1}}{\Delta_i} \cdot \frac{\Delta_i}{\mu_i - \mu_{c}} \cdot \frac{1 - \mu_{c}}{1 - \mu_{k+1}}
				\end{eqnarray*}
			More suggestively, we can write the above as
			\begin{eqnarray} \label{suggestExp}
			 \frac{1 - \mu_{k+1}}{\Delta_i} \cdot \frac{(1-\mu_{k+1}) - (1-\mu_i)}{(1-\mu_{c}) - (1-\mu_i)} \cdot \frac{1-\mu_{c}}{1-\mu_{k+1}}
			\end{eqnarray}
			As soon as $(1- \mu_{c}) > 2(1- \mu_{i})$,  Equation~\ref{suggestExp} is bounded by 
			\begin{eqnarray*}
			(\mu_i - \mu_{c})^{-1} (1-\mu_{c}) = \frac{1 - \mu_{k+1}}{\Delta_i} \cdot \frac{(1-\mu_{k+1}) - (1-\mu_i)}{\frac{1}{2}(1-\mu_{c})} \cdot \frac{1-\mu_{c}}{1-\mu_{k+1}} &=& \frac{2 ((1-\mu_{k+1}) - (1-\mu_i))}{\Delta_i} \\ 
			&\le& 2\frac{1 - \mu_{k+1}}{\Delta_i}
			\end{eqnarray*}
			Hence, 
			\begin{eqnarray*}
			\frac{1}{\mu_i - \mu_{c}} \cdot H_{i,c,t}\left( 1 + (1-\mu_{c})\Err_{i,c,t}  \right) &=& \frac{H_{i,c,t}}{\mu_i - \mu_{c}} + \frac{1-\mu_c}{\mu_i - \mu_c} \cdot \Err_{i,c,t} H_{i,c,t}   \\
			 &\le& \frac{H_{i,c,t}}{\Delta_i} + \frac{2H_{i,c,t}}{\Delta_i}\left( (1-\mu_{k+1})\Err_{i,c,t}  \right) \\
			  &=& \frac{H_{i,c,t}}{\Delta_i}\left( 1 + 2(1-\mu_{k+1})\Err_{i,c,t}  \right)\\
			  &\le& \frac{H_{i,c,t}}{\Delta_i}\left( 1 + 2\kappa_2  \right)
			\end{eqnarray*}
			where the last line follows from Lemma~\ref{kappaClaim} with $j = k+1$.
		\end{proof}

		\begin{proof}[Proof of Lemma~\ref{kappaClaim}]

			$\Siorc$ has the same distribution $\Siandc \cup y$, where $y \sim \Unif[U' - \Siandc - \{i,c\},1]$. If $Y \sim \text{Bernoulli}(\mu_y)$ then
			\begin{eqnarray}
			\Expt{\calL(\Siorc)}= \Exp{(1-Y)\calL(\Siandc)} = \Expt{\Expt{1-Y \big{|} \Siandc }\cdot\calL(\Siandc)}
			\end{eqnarray}
			Since $j \in [k+1]$, $\mu_j \ge \mu_{\ell}$ for all $\ell \notin [k]$, and thus $(1-\mu_{\ell}) \ge (1-\mu_j)$ for all $\ell \notin [k]$. It thus follows that
			\begin{eqnarray*}
			\Expt{1-Y \big{|} \Siandc} &=& \frac{1}{|U' - \{i,c\} - \Siandc|} \sum_{\ell \in U' - \{i,c\} - \Siandc} (1 - \mu_{\ell}) \\
			&\ge& \frac{1}{|U' - \{i,c\} - \Siandc|}\sum_{\ell \in U' - \{i,c\} - \Siandc - [k]} (1 - \mu_{\ell})\\
			&\ge& \frac{1}{|U' - \{i,c\} - \Siandc|} \sum_{\ell \in U' - \{i,c\} - \Siandc - [k]} (1 - \mu_j)\\
			&=& \frac{|U' - \{i,c\} - \Siandc - [k]|}{|U' - \{i,c\} - \Siandc|} (1 - \mu_j)
			\end{eqnarray*}
			If $r \in [k] \cap U' = U \cup B$, then we must have $r \in U$, since $B \cap [k] = \emptyset$ by assumption. This implies that $|U' - \{i,c\} - \Siandc - [k]| \ge |U' - \{i,c\} - \Siandc| - \min\{k,|U|\}  $. Using the fact that $|U' - \{i,c\} - \Siandc| = |U'| - k^{(1)}$, and that $k^{(1)} = \min\{k,|U|\}$, we conclude that 
			\begin{eqnarray}
			\Expt{1-Y \big{|} \Siandc} &\ge& (1-\mu_j)\frac{|U'| - k^{(1)} - \min \{k, |U|\}}{|U'| - k^{(1)}}  \\
			&=& (1-\mu_j)\frac{|U'| - 2 k^{(1)} }{|U'| - k^{(1)}} 
			\end{eqnarray}
			Thus, this entails that $ \Exp[\calL(\Siorc)] \ge (1-\mu_j)\frac{|U'| - 2k^{(1)}}{|U'| - k^{(1)}}\Exp[\calL(\Siandc)]$, and hence
			\begin{eqnarray}
			(1-\mu_j)\Err_{i,c,t} &=& \frac{k^{(1)} - 1}{|U'| - k^{(1)}} \cdot \frac{(1-\mu_j)\Exp{\calL(\Siandc)}}{\Exp{\calL(\Siorc) }}  \\
			&\le&  \frac{k^{(1)} - 1}{|U'| - k^{(1)}} \cdot \frac{|U'| - k^{(1)}}{|U'| - 2k^{(1)}}\\
			&=&  \frac{k^{(1)} - 1}{|U'| - 2k^{(1)}} := \kappa_2
			\end{eqnarray}
			as needed.
		\end{proof}

	\subsubsection{Controlling $\kappa_1$ and $\kappa_2$}
		\begin{proof}[Proof of Claim~\ref{BanditBalanceClaim}]
			For ease of notation, drop the dependence on the round $t$ and the definitions $\kappa_1 =  1 - \frac{k^{(1)} - 1}{|U'| - 1}$ and $\kappa_2 = \frac{k^{(1)} - 1}{|U'| - 2k^{(1)}}$. Noting that $|U'| = |B| + |U|$, we see that if $\kappa_1 \ge 1/2$ is desired, we require that
			\begin{eqnarray}
			\kappa_1 \ge 1/2 \iff |U'| - 1 \ge 2(k^{(1)} - 1) \iff |B| \ge 2 k^{(1)} - |U| - 1
			\end{eqnarray}
			Whereas
			\begin{eqnarray}
			\kappa_2 \le 2 \iff 2(|U'| - 2k^{(1)}) \ge k^{(1)} - 1 \iff |B| \ge \frac{5}{2} k^{(1)} - |U| - \frac{1}{2}
			\end{eqnarray}
			Hence $\kappa \le 2 \implies \kappa_1 \le 1/2$, and the above display makes it clear that the choice of $B$ in Algorithm~\ref{Balance} ensures that this holds. To verify the second condition, note that when $|B| = 0$, then $|U'| = |U|$. When $|B| > 0$, we have 
			\begin{eqnarray}
			|B| &=& \ceil{\frac{5k^{(1)}}{2} - |U| - \frac{1}{2}} \le \frac{5k^{(1)}}{2} - |U| 
			\end{eqnarray}
			so that $|U'| = |U| + |B| \le \frac{5}{2} \min\{|U|,k\}$. Finally, in order to always sample a balance set $B \subseteq R$, we need to ensure that at each round, $|R| \ge |B|$. Again, we may assume that $|B| > 0$, so that $|U| + |B| \le \frac{5k}{2}$. Using the facts that  $|R| + |A| + |U| = n$ (every item is rejected, accepted, and undecided) and $|A| \le k - 1$ ($k$ accepts ends the algorithm), we have $|R| \ge n - |U| - (k-1)  \ge n - |U| - (k-1)$.  But $n - |U| - (k-1) \ge |B| \iff n \ge (k-1) + |U'| \ge \frac{7k}{2}$, as needed. 
		\end{proof}

\section{Concentration Proofs for Section~\ref{GeneralAnalySec} }
	\subsection{An Empirical Bernstein}
		The key technical ingredient is an empirical version of Bernstein's inequality, which lets us build variance-adaptive confidence intervals:
			\begin{thm}[Modification of Theorem 11 in \cite{maurer2009empirical} ]\label{EmpiricalBernstein} Let $Z:=(Z_1,\dots,Z_n)$ be a sequence of independent random variables bounded by $[0,1]$. Let $\bar{Z}_n = \frac{1}{n} \sum_{i} Z_i$, $\bar{Z} := \Exp[\bar{Z}_n]$, let $\Var_n[Z]$ denote the empirical variance of $Z$, $\frac{1}{n-1}\sum_{i=1}^n (Z_i^2 - \bar{Z}_n^2)$, and set $\Var[Z] := \Exp[\Var_n[Z]]$. Then, with probability $1-\delta$,
			\begin{eqnarray}
			\left|\bar{Z} - \bar{Z}_n\right| &\le& \sqrt{\frac{2\Var_n[Z] \log(4/\delta)}{n}} +  \frac{8\log(4/\delta)}{3(n-1)} \\
			&\le& \sqrt{\frac{2\Var[Z] \log(4/\delta)}{n}} + \frac{14\log(4/\delta)}{3(n-1)}
			\end{eqnarray}
			\end{thm}
			
			The result follows from Bernstein's Inequality, and the following concentration result regarding the square root of the empirical variance.

			\begin{lem}[Theorem 10 in~\cite{maurer2009empirical}] In the set up of Theorem~\ref{EmpiricalBernstein}, 
					\begin{eqnarray}\label{SqrtEq}
					\left|\sqrt{\Exp[\Var_n[Z]]} - \sqrt{\Var_n[Z]}\right| \le \sqrt{\frac{2 \log(2/\delta)}{n-1}}
					\end{eqnarray}
					hold with probability $1 - \delta$.
			\end{lem}

			\begin{proof}[Proof of Theorem~\ref{EmpiricalBernstein}]
					The argument follows the proof of Theorem 11 in \cite{maurer2009empirical}. Let $W := \frac{1}{n}\sum_{i=1}^n \Var[Z_i]$. It is straightforward to verify that $W \le \Exp[\Var_n[X]]$, and hence Bernstein's inequality yields that, with probability $1-\delta$,
					\begin{equation}
					\begin{aligned} \label{Bernstein}
					\left|\frac{1}{n}\sum_{i=1}^n Z_i - \Exp[Z_i] \right| &\le \sqrt{\frac{2W \log(4/\delta)}{n}} + \frac{2 \log (4/\delta)}{3n}\\
					&\le \sqrt{\frac{2\E[ \Var_n[Z] ] \log(4/\delta)}{n}} + \frac{2 \log (4/\delta)}{3n} \\
					&\le \sqrt{\frac{2 \log(4/\delta)}{n}}\cdot\sqrt{\Var_n[Z]} + \frac{2 \log (4/\delta)}{\sqrt{n(n-1)}} + \frac{2 \log (4/\delta)}{3n} \\
					&< \sqrt{\frac{2 \Var_n[Z] \log(4/\delta)}{n}} + \frac{8 \log (4/\delta)}{3(n-1)} \\
					&< \sqrt{\frac{2 \E[ \Var_n[Z] ] \log(4/\delta)}{n}} + \frac{14 \log (4/\delta)}{3(n-1)}
					\end{aligned}
					\end{equation}
					which completes the proof.
					
			\end{proof}

		In our algorithm, the confidence intervals $\hat{C}_{i,t}$ depend on sample variances, and are thus random. To insure they are bounded above, we define a confidence parameter $C_{i,t}$ which depends on the true (but unkown) stagewise variance parameter
		\begin{eqnarray}
			C_{i,t} := \sqrt{\frac{2 V_{i,t} \log(8nt^2/\delta)}{T(t)}} + \frac{14 \log (8 nt^2/\delta)}{3 (T(t)-1)}
		\end{eqnarray}
		We extend our Empirical Bernstein bound to a union bound over all rounds $t \in \{1,2,\dots\}$, showing that, uniformly over all rounds, $\hat{C}_{i,t}$ is a reasonable confidence interval and never exceeds $C_{i,t}$:

		\begin{lem}[Stagewise Iterated Logarithm Bound for Empirical Bernstein]\label{IteratedLog} Let
		\begin{eqnarray}
		\mathcal{E} := \{ \cap_{t = 1}^{\infty} \cap_{i=1}^n \{ \left| \hat{\mu}_{i,t} - \bar{\mu}_{i,t} \right| \le \hat{C}_{i,t} \le C_{i,t}\}\}
		\end{eqnarray}
		Then $\Pr(\mathcal{E}) \ge 1-\delta$.
		\end{lem}
		\begin{proof}
		Let $\mathcal{E}_{i,t}$ denote the event that $\{\left| \hat{\mu}_{i,t} - \bar{\mu}_{i,t} \right| \le \hat{C}_{i,t} \le C_{i,t}\}$. Conditioned on any realization of the data $\mathcal{D}_{t}$ at stage $t$, an application of Theorem~\ref{EmpiricalBernstein} shows that $\Pr( \mathcal{E}_{i,t} \big{|} \mathcal{D}_t) \le \frac{\delta}{ 2nt^2}$. Integrating over all such realizations, $\Pr( \mathcal{E}_{i,t}) \le \frac{\delta}{ 2nt^2}$. Finally, taking a union bound over all stages $t$ and arms $i \in [n]$ shows that 
		\begin{eqnarray}
		\Pr(\mathcal{E})\le \sum_{t= 1}^{\infty} \sum_{i = 1}^n \Pr(\mathcal{E}_{i,t}) \le \sum_{t= 1}^{\infty} \sum_{i = 1}^n \frac{\delta}{ 2nt} = \frac{\delta}{2} \sum_{t= 1}^{\infty} t^{-2} \le \delta
		\end{eqnarray}
		\end{proof}
		
		We now invert the Iterated Logarithm via
		\begin{lem}[Inversion Lemma]\label{InvLem} For any $\Delta > 0$ and $t \ge 2$, $C_{i,t} \le \Delta$ as long as 
		\begin{eqnarray}
		T \ge \left(\frac{16 V_{i,t}}{\Delta^2} + \frac{14}{\Delta}\right) \log \left(\frac{24 n  }{\delta}\log\left(\frac{12 n }{\delta}\left(\frac{16 V_{i,t}}{\Delta^2} + \frac{14}{\Delta}\right) \right)  \right) 
		\end{eqnarray}
		\end{lem}

		\begin{proof}
			It suffices to show that $\sqrt{\frac{2 V_{i,t} \log(8nt^2/\delta)}{T(t)}} \le \Delta/2 $ and $ \frac{14 \log (8 nt^2/\delta)}{3 (T(t)-1)} \le \Delta/2$. Since $t^2 = (\log_2(T))^2 \le (\log_2 e \log(T))^2$, it suffices that 
			\begin{eqnarray*}
			\frac{8  V_{i,t} \log(8n \log_2^2 e \log^2 (T(t))/\delta)) }{\Delta^2 T(t)}  \le 1 & \text{ and } & \frac{28  \log(8n  \log_2 e \log^2 (T(t))/\delta) }{3 \Delta (T(t)-1)} \le 1
			\end{eqnarray*}
			As long as $t \ge 2$, so that $T(t) \ge e$, it suffices that
			\begin{eqnarray*}
			\frac{16 V_{i,t} \log(8n \log_2 e \log (T(t))/\delta)) }{\Delta^2 T(t)}  \le 1 & \text{ and } & \frac{14 \log(8n  \log_2 e \log (T(t))/\delta) }{\Delta T(t)} \le 1
			\end{eqnarray*}
			Let $\alpha_1 = 16 V_{i,t}/ \Delta^2$, $\alpha_2 = 14 / \Delta$  and  $\beta = 8n\log_2e/\delta < 12 n/\delta$. Then both inequalities take the form 
			\begin{eqnarray}
			\alpha_p \log(\beta \log(T))/T \le 1
			\end{eqnarray}
			where we simplify $T(t) = T$. Using the inversion 
			\begin{eqnarray}
			T \ge \alpha \log (2\beta \log(\alpha \beta)) \implies \alpha \log (\beta \log(T))/T \le 1 
			\end{eqnarray}
			we obtain that it is sufficient for $T \ge (\alpha_1 + \alpha_2) \log (2 \beta \log (\alpha_1 + \alpha_2)) \ge \max_p \alpha_p \log (2\beta \log(\alpha_p \beta))$, or simply
			\begin{eqnarray}
			T \ge \left(\frac{16 V_{i,t}}{\Delta^2} + \frac{14}{\Delta}\right) \log \left(\frac{24 n  }{\delta}\log\left(\frac{12 n }{\delta}\left(\frac{16 V_{i,t}}{\Delta^2} + \frac{14}{\Delta}\right) \right)  \right) 
			\end{eqnarray}
		\end{proof}

		\subsection{Proof of Theorem~\ref{SuccElimGuar}}
			We show that Theorem~\ref{SuccElimGuar} holds as long as the event $\mathcal{E}$ from Lemma~\ref{IteratedLog} holds. The definition of $\mathcal{E}$ and Algorithm~\ref{SuccElim} immediately imply that no arms in $[k]$ are rejected, and no arms in $[n] - [k]$ are accepted. To prove the more interesting part of the theorem, fix an index $i \in U_t$, and define
			\begin{eqnarray}
			 \mathcal{C}(i) := \begin{cases} \{j \in U_t, j > k\} & i \le k\\
			\{j \in U_t, j \le k\} & i > k
			\end{cases}
			\end{eqnarray}
			Also, let $c_{i} = \arg\min_{j \in \mathcal{C}(i)} |\mu_i - \mu_j|$.
			We can think of $\mathcal{C}(i)$ as the set of all arms competing with $i$ for either an accept or reject, and $c_{i}$ as the competitor closest $i$ in mean. For $i > k$ to be rejected, it is sufficient that, for all $j \in \mathcal{C}(i)$, $\min_{j \in \mathcal{C}(i)} \hat{\mu}_{j,t} - \hat{C}_{j,t} \ge \hat{\mu}_{i,t} + \hat{C}_{j,t}$. Under $\mathcal{E}$, $\hat{\mu}_{j,t} - \hat{C}_{j,t} \ge \bar{\mu}_{j,t} - 2C_{j,t}$, and $\bar{\mu}_{i,t} \le \bar{\mu}_{i,t} + 2C_{i,t}$, so that it is sufficient for 
			\begin{eqnarray}
			\forall j \in \mathcal{C}(i) :\bar{\mu}_{j,t} - \bar{\mu}_{i,t} \ge 2(C_{i,t} + C_{j,t})
			\end{eqnarray}
			Analogously, for $i \le k$, $i$ is accepted under $\mathcal{E}$ as long as $\forall j \in \mathcal{C}(i): \bar{\mu}_{i,t} - \bar{\mu}_{j,t} \ge 2(C_{i,t} + C_{j,t})$. Defining $\Delta_{i,j,t} := |\bar{\mu}_{i,t} - \bar{\mu}_{j,t} |$, we subsume both cases under the condition
			\begin{eqnarray}
			\forall j \in \mathcal{C}(i) :\Delta_{i,j,t} \ge 2(C_{i,t} + C_{j,t})
			\end{eqnarray}
			for which it is sufficient to show that
			\begin{eqnarray}\label{sufficient}
			\forall j \in \mathcal{C}(i) : C_{i,t} \le \Delta_{i,j,t}/4 & \text{ and } & C_{j,t} \le \Delta_{i,j,t}/4
			\end{eqnarray}
			To this end define 
			\begin{eqnarray}
			\tau_{i,j,t}^{(1)} = \frac{256  V_{i,t}}{\Delta_{i,j,t}^2} + \frac{56  }{\Delta_{i,j,t}} & \text{and } & \tau_{i,j,t}^{(2)} := \frac{256  V_{j,t}}{\Delta_{i,j,t}^2} + \frac{56}{\Delta_{i,j,t}}
			\end{eqnarray}
			We now show that $\tau_{i,t} = \max_{j \in \mathcal{C}(i)} \max\left\{\tau_{i,j,t}^{(1)},\tau_{i,j,t}^{(2)}\right\}$, which by Lemmas~\ref{IteratedLog} and~\ref{InvLem} implies that Equation~\ref{sufficient} will holds as long as
			\begin{eqnarray}
			T \ge \tau_{i,t} \log \left(\frac{24n}{\delta}\log\left(\frac{12 n \tau_{i,t}}{\delta}\right) \right)  
			\end{eqnarray} 
			Now, we bound $\tau_{i,t}$. Note that $\Delta_{i,j,t} \ge \Delta_{i,c_i,t}:= \Delta_{i,t} $ for all $j \in \mathcal{C}(i)$. This implies that $\max_{j \in \mathcal{C}(i)} \tau_{i,j,t}^{(1)} \le \frac{256 V_{i,t}}{\Delta_{i,t}^2} + \frac{56 }{\Delta_{i,t}}$. On the other hand, it holds that 
			\begin{eqnarray}
			\max_{j \in \mathcal{C}(i)} \tau_{i,j,t}^{(2)} &\le& 256 \max_{j \in \mathcal{C}(i)} \left(\frac{ V_{j,t}}{\Delta_{i,j,t}^2}\right) + \frac{56 }{\Delta_{i,t}} \le \frac{ 256 V_{c_i,t}}{\Delta_{i}^2} + \frac{56 }{\Delta_{i,t}}
			\end{eqnarray}
			where the second inequality invokes the following lemma.
			\begin{lem}\label{VarianceTechnicalLemma} For $i \in \{1,2,3\}$, $Z_i \sim \text{Bernoulli}(p_i)$, where either $p_1 < p_2 < p_3$ or $p_3 > p_2 > p_1$. Then,
			\begin{eqnarray}
			\frac{\Var[Z_2]}{(\Exp[Z_1 - Z_2])^2} \ge 	\frac{\Var[Z_3]}{(\Exp[Z_1 - Z_3])^2} 
			\end{eqnarray}
			\end{lem}

		\begin{proof}[Proof of Lemma~\ref{VarianceTechnicalLemma}]
		The desired inequality and conditions are invariant under the tranformation $p_i \mapsto 1 - p_i$ for $i \in \{1,2,3\}$, so we may assume without loss of generality that. $p_1< p_2< p_3 \in [0,1]$. Then $1 > p_1/p_2 > p_1/p_3$, which implies that 
		\begin{eqnarray*}
		\frac{1}{1 - p_1/p_2} \ge \frac{1}{1 - p_1/p_3} &\implies& \frac{p_2}{ p_2 - p_1} \ge \frac{p_3}{p_3 - p_1} \\
		&\implies& \frac{(1-p_2)p_2}{p_2-p_1} \ge \frac{(1-p_3)p_3}{ p_3 - p_1}\\
		&\implies& \frac{(1-p_2)p_2}{(p_2-p_1)^2} \ge \frac{(1-p_3)p_3}{(p_3 - p_1)^2}
		\end{eqnarray*}
		which is precisely the desired inequality. 
		\end{proof}

	\subsection{Proof of Lemma~\ref{EfficiencyLemma}\label{EfficiencyLemProof}}

		Let $e_t$ be denote the the ``efficiency'', so that, at round $t$, each call of uniform play for $s = 1,\dots, T(t)$ makes at most $e_t |U_t|$ queries. Furthermore, let $\tau_0$ denote the first time such that $|U_t| < k$. By assumption, we have that $e_t \le \frac{\alpha}{k}$ for $0 \le t < \tau_0$, and that $e_t |U_t| \le \alpha$ for $t \ge t_0$. Finally, let $\tau_{i}^* = \inf\{t: i \notin U_t\}$.	Then, the total number of samples we collect is 
		\begin{eqnarray}
		\sum_{t=0}^{\infty} e_t|U_t| T(t) &=& \sum_{t=0}^{\tau_0 - 1} e_t|U_t| T(t) + \sum_{t = \tau_0}^{\infty} e_t |U_t| T(t) \\
		&\le& \frac{\alpha}{k}\sum_{t=0}^{\tau_0 - 1} |U_t| T(t) + \alpha \sum_{t = \tau_0}^{\infty} \I(U_t \ne \emptyset) T(t) 
		\end{eqnarray}
		The first sum can be re-arranged via
		\begin{eqnarray}
		\sum_{t=0}^{\tau_0 - 1} |U_t| T(t) &=& \sum_{t= 0}^{\tau_0 - 1} \left(\sum_{i = 1}^n \I(i \in u_t)\right) T(t) = \sum_{i = 1}^n \sum_{t= 0}^{\tau_0 + 1} \I(i \in U_t) T(t)\\
		 &\le& \sum_{i = 1}^n \sum_{t= 0}^{\infty} \I(i \in U_t) T(t) \le \sum_{i = 1}^n 2^{\tau^*_i + 1} 
		\end{eqnarray}
		whereas the second sum is bounded above by $\sum_{t = \tau_0}^{\infty} \I(U_t \ne \emptyset) T(t) \le 2^{\max_j \tau^*_j + 1}$. Hence, 
		\begin{eqnarray}
		\sum_{t=0}^{\infty} e_t|U_t| T(t) \le 2\alpha(2^{\max_j \tau^*_j} + \frac{1}{k}\sum_{i=1}^n  2^{\tau^*_i})
		\end{eqnarray}
		Finally, let $T^*_i := 2^{\tau^*_i}$, and let $\sigma() :[n] \to n$ denote a permutation such that $T^*_{\sigma(1)} \ge T^*_{\sigma(2)} \dots T^*_{\sigma(n)}$. Then, a straight forward manipulation of the above display yields that
		\begin{eqnarray}
		\sum_{t=0}^{\infty} e_t|U_t| T(t) \le 2\alpha(2T_{\sigma(1)}^* + \frac{1}{k}\sum_{i=k+1}^n  T^*_{\sigma(i)})
		\end{eqnarray}
		since $\frac{1}{k}\sum_{i=1}^k T_{\sigma(i)}^* \leq T_{\sigma(1)}^*$.

\section{Dependent Lower Bound Proof}
Recall that we query subsets of $S \subset \calS := \binom{[n]}{k}$. Let $T_S$ denote the number of times a given subset $S$ is queried, and note that the expected sample complexity is simply:
\begin{align*}
\sum_{S \in \calS} \Exp[T_S]
\end{align*}

Further, let  $d(x,y)$ denote the KL-divergence between two independent, Bernoulli random variables with means $x$ and $y$, respectively. We first need a technical lemma, whose proof we defer the end of the section:

\begin{lemma}\label{bernoulli_kl_bounds}
Let $d(x,y) = x \log(\frac{x}{y}) + (1-x) \log(\frac{1-x}{1-y})$. Then 
\begin{align}
\frac{(y-x)^2/2}{ \sup_{z \in [x,y] }z(1-z) } \leq d(x,y) \leq \frac{(y-x)^2/2}{x(1-x) - [(y-x)(2x-1)]_+} \leq \frac{(y-x)^2/2}{\min\{x(1-x),y(1-y)\}} 
\end{align}
\end{lemma}

We break the proof up into steps. First we construct the dependent measure $\nu$ that is $(k-1)$-wise independent, meaning that for any subset $S \in \binom{[n]}{k}$, any subset of size $(k-1)$ of $S$ behaves like independent arms. The construction makes it necessary to consider each set of $k$ individually. To obtain the lower bounds we appeal to a change of measure argument (see \cite{kaufmann2014complexity} for details) that proposes an alterantive measure $\nu'$ in which a different subset is best than that subset that is best in $\nu$, and then we calculate the number of measurements necessary to rule out $\nu'$. The majority of the effort goes into 1) computing the gap between the best and all other subsets and 2) computing the KL divergences between $\nu$ and the alterantive measures $nu'$ under the bandit and semi-bandit feedback mechanisms.

\noindent\textbf{Step 1: Construct $\nu$}: 

	Fix $p \in [0,1]$ and $\mu \in [0,1/2]$. Let $X = (X_1,\dots,X_n)$ be distributed according to $\nu$. Define the independent random variables $Y$ as Bernoulli$(p)$, $Z_i$ as Bernoulli$(1/2)$, and $U_i$ as Bernoulli$(2\mu)$ for all $i \in [n]$. For $i>1$ let $X_i = Z_i U_i$ and let
	\begin{align*}
	X_1 = U_1 \widetilde{Z}_1 \quad\quad \text{where} \quad\quad \widetilde{Z}_1 = \begin{cases} 1 + \oplus_{i=2}^k Z_i & \text{ if } Y=1 \\  Z_1 & \text{ if } Y=0 \end{cases}
	\end{align*}
	where $\oplus$ denotes modular-2 addition. Note that $\E_{\nu}[ X ] = \mu \mathbf{1}$ since 
	\begin{align*}
	\E_{\nu_1}[X_1] = 2 \mu \left[ p \ \P_{\nu}\left( 1 + \oplus_{i=2}^k Z_i = 1 \right) + (1-p) \tfrac{1}{2} \right] = \mu
	\end{align*} 
	and the calculation for $\E[ X_i ]$ for $i>1$ are immediate by independence. Henceforth, denote $S^* = \{1,\dots,k\}$.

\noindent\textbf{Step 2: Relevant Properties of $\nu$}:

\begin{enumerate}
	\item \emph{Any subset of arms $S$ which doesn't contain all of $S^*$ are independent}. If $Y=0$ then the claim is immediate so assume $Y=1$. We may also assume that $1 \in S$, since otherwise the arms are independent by construction. Finally, we remark that even when $1 \in S$ and $Y  = 1$, all arms in $S$ are conditionally independent given $\{ Z_{i}: i \in S \cap S^*\}$. Thus, it suffices to verify that $\{ Z_{i}: i \in S \cap S^* - 1\} \cup \{\widetilde{Z}_1\}$ have a product distribution. To see this, note that $\{ Z_{i}: i \in S \cap S^* - 1\}$ is a product distribution, so it suffices to show that $\widetilde{Z}_1$ is independent of $\{ Z_{i}: i \in S \cap S^* - 1\}$. Write $\widetilde{Z}_1 = 1 + \oplus_{i \in S^* \setminus 1} Z_i = 1 \oplus_{i \in S^* \cap S} Z_i \oplus_{i \in S^* \setminus  S} Z_i$. The sum over $Z_i$ not in $S^*$, $\oplus_{i \in S^* \setminus  S} Z_i$, is Bernoulli(1/2), and independent of all the $Z_i$ for which $i \in S^* \cap S$. Thus, conditioned on any realization of $\{Z_i : i \in S \cap S^*\}$, $\widetilde{Z}_1$ is still Bernoulli(1/2), as needed.  
	\item \emph{The distribution of $\nu$ is invariant under relabeling of arms in $S^*$, and under relabeling of arms $[n] \setminus S^*$.} The second part of the statement is clear. Moreover, since the arms in $[n] \setminus S^*$  are independent of those $S^*$, it suffices to show that the distribution of arms in $S^*$ are invariant under relabeling. Using the same arguments as above, we may reduce to the case where $Y = 1$, and only verify that the distribution of $\{\widetilde{Z}_1\} \cup \{Z_i : i \in S^* - 1\}$ is invariant under relabeling. 

	To more easily facilliate relabeling, we adjust our notation and set $\widetilde{Z}_i = Z_i$ for $i \in S^* \setminus 1$ (recall again that $Y = 1$, so there should be no ambiguity). Identify $S^* \equiv [k]$, fix  $t \in \{0,1\}^k$, and consider any permutation $\pi: [k] \to [k]$. We have
	\begin{eqnarray*}
	&&\Pr((\widetilde{Z}_{\pi(1)},\dots,\tilde{Z}_{\pi(k)}) = t) \\
	&=& \Pr((\widetilde{Z}_{\pi(1)}  = t_1 \big{|} \widetilde{Z}_{\pi(2)},\dots,\tilde{Z}_{\pi(k)}) = t_2,\dots,t_k) ) \cdot \Pr(\widetilde{Z}_{\pi(2)},\dots,\tilde{Z}_{\pi(k)}) = t_2,\dots,t_k) 
	\end{eqnarray*}
	Using our adjusted notation, the relation between between $\widetilde{Z}_i$'s becomes $\widetilde{Z}_1 = 1 \oplus_{i \in S^* - 1} \widetilde{Z}_i$. This constraint is deterministic (again, $Y = 1$) and can be rewritten as $\oplus_{i \in S^*} \widetilde{Z} = 1$, which is invariant under-relabeling. Hence,  $ \Pr(\widetilde{Z}_{\pi(1)}  = t_1 \big{|} (\widetilde{Z}_{\pi(2)},\dots,\tilde{Z}_{\pi(k)}) = (t_2,\dots,t_k) ) = \I( \oplus_{i=1}^k t_{i} = 1)$. Moreover, we demonstrated above that, for any set $S$ not containing $S^*$, $\{ Z_{i}: i \in S \cap S^* - 1\} \cup \{\widetilde{Z}_1\}$ have a product distribution of $k-1$ Bernoulli(1/2) random variables. In our adjusted notation, this entails that $\Pr((\widetilde{Z}_{\pi(2)},\dots,\tilde{Z}_{\pi(k)}) = t_2,\dots,t_k) = 2^{-(k-1)}$. Putting things together, we see that
	\begin{eqnarray}
	\Pr((\widetilde{Z}_{\pi(1)},\dots,\tilde{Z}_{\pi(k)}) = t) =2^{-(k-1)} \I( \oplus_i t_i = 1)
	\end{eqnarray}
	which does not dependent on the permutation $\pi$.
	\end{enumerate}

\noindent\textbf{Step 3: Computation of the Gap under $\nu$}
	
	Note that if $S \neq S^*$ then
	\begin{align*}
	\E_\nu[ \max_{i \in S} X_i ] = \E_\nu[ \max_{i \in S}  Z_i U_i ] &=  \P\left( \cup_{i \in S} \{ Z_i =1, U_i  = 1\} \right) = 1 - \P_\nu( \cap_{i \in S} \{ Z_i =1, U_i  = 1\}^c ) \\
	&= 1 - \prod_{i \in S} \P_\nu(  \{ Z_i =1, U_i  = 1\}^c ) = 1 - \prod_{i \in S} ( 1- \P_\nu(  Z_i =1, U_i  = 1  ) ) \\
	&= 1 - \prod_{i \in S} ( 1- \P_\nu(  Z_i =1 ) \P( U_i  = 1  ) ) = 1 - ( 1- \mu)^k.
	\end{align*}
	Otherwise,
	\begin{align*}
	\E_\nu[ \max_{i \in S^*} X_i ] &= \E_\nu[ \max_{i \in S^*} X_i | Y = 1 ] \ p +  \E_\nu[ \max_{i \in S^*} X_i | Y = 0 ] \ (1- p) \\
	&= \E_\nu[ \max_{i \in S^*} X_i | Y = 1 ] \ p +  \left[ 1 - ( 1- \mu)^k \right] \ (1- p) 
	\end{align*}
	where
	\begin{align*}
	\E_\nu[ \max_{i \in S^*} X_i | Y = 1 ]  &=  1 -  \P( \max_{i \geq 1} U_i Z_i = 0 ) \\
	&=  1 -  \P( \max_{i \geq 1} U_i Z_i = 0, \oplus_{i > 1} Z_i = 0 )  - \P( \max_{i \geq 1} U_i Z_i = 0, \oplus_{i > 1} Z_i = 1 ) \\
	&=  1 -   (1-2\mu) \P( \max_{i > 1} U_i Z_i = 0, \oplus_{i > 1} Z_i = 0 )  - \P( \max_{i > 1} U_i Z_i = 0, \oplus_{i > 1} Z_i = 1 ) \\
	&=  1 -    \P( \max_{i > 1} U_i Z_i = 0 ) + 2 \mu  \P( \max_{i > 1} U_i Z_i = 0, \oplus_{i > 1} Z_i = 0 )  \\
	&=  1 -  (1-\mu)^{k-1} + 2 \mu  \P( \max_{i > 1} U_i Z_i = 0, \oplus_{i > 1} Z_i = 0 )  
	\end{align*}
	and
	\begin{align*}
	\P( \max_{i > 1} U_i Z_i = 0&, \oplus_{i > 1} Z_i = 0 ) = \sum_{\ell=0}^{\lfloor \tfrac{k-1}{2} \rfloor} \P\left( \max_{i > 1} U_i Z_i = 0, \sum_{i>1} Z_i = 2\ell \right) \\
	&= \sum_{\ell=0}^{\lfloor \tfrac{k-1}{2} \rfloor} \P\left( \max_{i > 1} U_i Z_i = 0 \bigg| \sum_{i>1} Z_i = 2\ell \right) \binom{k-1}{2\ell} 2^{-k+1} \\
	&= \sum_{\ell=0}^{\lfloor \tfrac{k-1}{2} \rfloor}  (1-2\mu)^{2 \ell} \binom{k-1}{2\ell} 2^{-k+1} \\
	&= \frac{2^{-(k-1)}}{2} \left( ( (1-2\mu) + 1)^{k-1} + (-1)^{k-1}( (1-2\mu) - 1)^{k-1}  \right) \\
	&= \frac{1}{2} \left( ( 1- \mu)^{k-1} + \mu^{k-1}  \right) \\
	\end{align*}
	since
	\begin{align*}
	( (1-2\mu) + 1)^{k-1} &= \sum_{j=0}^{k-1}  (1-2\mu)^{j}  \binom{k-1}{j} \\
	&= \sum_{\ell=0}^{\lfloor \tfrac{k-1}{2} \rfloor}  (1-2\mu)^{2 \ell} \binom{k-1}{2\ell}  + \sum_{\ell=0}^{\lfloor \tfrac{k-1}{2} \rfloor}  (1-2\mu)^{2 \ell+1}  \binom{k-1}{2\ell+1}  
	\end{align*}
	and
	\begin{align*}
	( (1-2\mu) - 1)^{k-1} &= \sum_{j=0}^{k-1}  (-1)^j (1-2\mu)^{k-1-j} \binom{k-1}{j} \\
	&= \sum_{\ell=0}^{\lfloor \tfrac{k-1}{2} \rfloor}  (1-2\mu)^{k-1-2 \ell} \binom{k-1}{2\ell}  - \sum_{\ell=0}^{\lfloor \tfrac{k-1}{2} \rfloor}  (1-2\mu)^{k-2-2 \ell}  \binom{k-1}{2\ell+1}  \\
	&= (-1)^{k-1}\sum_{\ell=0}^{\lfloor \tfrac{k-1}{2} \rfloor}  (1-2\mu)^{2 \ell} \binom{k-1}{2\ell}  - (-1)^{k-1}\sum_{\ell=0}^{\lfloor \tfrac{k-1}{2} \rfloor}  (1-2\mu)^{2 \ell+1}  \binom{k-1}{2\ell+1} .
	\end{align*}
	Putting it all together we have
	\begin{equation}
	\begin{aligned}
	\E_\nu[ \max_{i \in S^*} X_i ] &= [ 1 -  (1-\mu)^{k-1} + \mu \left( ( 1- \mu)^{k-1} + \mu^{k-1}  \right) ] \ p +  \left[ 1 - ( 1- \mu)^k \right] \ (1- p) \\
	&= [ 1 -  (1-\mu)^{k} + \mu^{k}  ] \ p +  \left[ 1 - ( 1- \mu)^k \right] \ (1- p)  \\
	&=  \left[ 1 - ( 1- \mu)^k \right]  +  \mu^{k}  p 
	\end{aligned}
	\end{equation}

	%
	Thus, $\Delta = p \mu^k$ which is maximized at $\mu = \frac{1}{2}$ achieving $\Delta = p 2^{-k}$.

\noindent\textbf{Step 4: Change of measure}: Consider the distribution $\nu$ that is constructed in Step 1 that is defined with respect to $S^* = \{1,\dots,k\}$. For all $S \in \calS$ we will now construct a new distribution $\nu^S$ such that $\E_{\nu^S}[ \max_{i \in S} X_i ] > \E_{\nu^S}[ \max_{i \in S^*} X_i ] = \E_{\nu}[ \max_{i \in S^*} X_i ]$. We begin constructing $\nu^S$ identically to how we constructed $\nu$ but modify the distribution of $X_{\Sell}$ where $\Sell = \arg\min \{ i: X_i, i \in S\} $. In essence $X_{\Sell}$ with respect to $S \in \mathcal{S}$ will be constructed identically to the construction of $X_1$ with respect to $S^* = \{1,\dots,k\}$ with the one exception that in place of $Y$ we will use a new random variable $Y^S$ that is Bernoulli$(p')$ where $p' > p$ (this is always possible as $p < 1$). 

Let $\nu(S)$ describe the joint probability distribution of $\nu$ restricted to the set $i \in S$. And for any $S \in \mathcal{S}$ let $\tau$ denote the projection of $\nu(S)$ down to some smaller event space. For example, $\tau \nu(S)$ can represent the Bernoulli probability distribution describing $\max_{i \in S} X_i$ under distribution $\nu$. By $(k-1)$-wise independence we have
\begin{align*}
 KL(  \nu(S') | \nu^S (S') ) = 0 \quad \forall S' \in \mathcal{S} \setminus S
\end{align*}
since $S$ and $S'$ differ by at least one element and $\nu(S^*) = \nu^S(S^*)$. Clearly, $KL( \tau \nu(S') | \tau \nu^S (S') ) = 0$ as well for all $S' \in \mathcal{S} \setminus S$. By assumption, any valid algorithm correctly identifies $S^*$ under $\nu$, and $S$ under $\nu^S$, with probability at least $1-\delta$. Thus, by Lemma 1 of \cite{kaufmann2014complexity}, for every $S \in \mathcal{S} \setminus S^*$ 
\begin{align*}
 \log( \tfrac{1}{2\delta} ) \leq \sum_{S' \in \mathcal{S}} \E_\nu[ T_{S'} ]  KL( \tau \nu(S') | \tau \nu^S (S') )  = KL( \tau \nu(S) | \tau \nu^S (S) )   \E_\nu[ T_S ] \ ,
\end{align*}
where we recall that $T_S$ is the number of times the set $S$ is pulled. Hence,
\begin{align*}
\E_\nu\left[ \sum_{S \in \mathcal{S} \setminus S^*}T_S \right]  &\geq  \sum_{S \in \calS \setminus S^*} \frac{ \log( \tfrac{1}{2\delta} ) }{ KL( \tau \nu(S) | \tau \nu^S (S) )    } \\
&= \frac{ \log( \tfrac{1}{2\delta} ) }{ KL( \tau \nu(S) | \tau \nu^S (S) )    } \left[  \binom{n}{k} -1 \right] \geq \frac{ \tfrac{2}{3}\log( \tfrac{1}{2\delta} ) }{ KL( \tau \nu(S) | \tau \nu^S (S) )    }   \binom{n}{k} 
\end{align*}
where the equality holds for any fixed $S \in \calS$ by the symmetry of the construction and the last inequality holds since $2 \leq k < n$,  $\binom{n}{k} -1 \geq \frac{2}{3} \binom{n}{k}$. It just remains to upper bound the KL divergence.

\noindent\textbf{Bandit feedback}: Let $\tau \nu(S)$ represent the Bernoulli probability distribution describing $\max_{i \in a} X_i$ under distribution $\nu$. Then by the above calculations of the gap we have 
\begin{align*}
\hspace{.5in}&\hspace{-.5in}KL( \tau \nu(S) | \tau \nu^S (S) ) = KL( 1-(1-\mu)^k | 1-(1-\mu)^k + p' \mu^k ) \\
&\leq \frac{ p'^2 \mu^{2k}/2 }{ (1-(1-\mu)^k ) (1-\mu)^k -2 p' \mu^k [\tfrac{1}{2}-(1-\mu)^k]_+ } \\
&\leq \frac{ p'^2 \mu^{2k}/2 }{ (1-(1-\mu)^k ) ((1-\mu)^k - p' \mu^k) } 
\end{align*}
by applying Lemma~\ref{bernoulli_kl_bounds} and noting that
\begin{align*}
(1-(1-\mu)^k )& (1-\mu)^k -2 p' \mu^k [\tfrac{1}{2}-(1-\mu)^k]_+ \\ 
&\geq \min\{ (1-(1-\mu)^k ) (1-\mu)^k, (1-(1-\mu)^k ) (1-\mu)^k - p' \mu^k (1- 2(1-\mu)^k)\} \\
&\geq \min\{ (1-(1-\mu)^k ) (1-\mu)^k, (1-(1-\mu)^k ) [ (1-\mu)^k - p' \mu^k ]\} \\
&\geq (1-(1-\mu)^k ) ( (1-\mu)^k - p' \mu^k )
\end{align*}
Finally, let $p' \rightarrow p$. Setting $\mu = 1- 2^{-1/k} \geq \frac{1}{2k}$ we have $(1-\mu)^k = 1/2$ and $\Delta \geq p (2k)^{-k}$ so that $\E_\nu\left[ \sum_{S \in \calS \setminus S^*}T_S \right]  \geq \frac{1}{3} \binom{n}{k} \Delta^{-2} \log( \tfrac{1}{2\delta} ).$

\noindent\textbf{Marked-Bandit feedback}: Let $\tau \nu(S)$ represent the distribution over $\perp \cup S$ under $\nu$ such that if $W \sim \tau \nu(S)$ then $W$ is drawn uniformly at random from $\arg\max_{i \in S} X_i$ if $\max_{i \in S} X_i=1$, and $W = \perp$ otherwise. By the permutation invariance property of $\nu$ described in Step 2, we have for any $S \in \binom{[n]}{k}-S_*$ and $i \in S$
\begin{align*}
\P_{\nu}( W = i | W \neq \perp ) = \P_{\nu^S}( W = i | W \neq \perp ) = \frac{1}{k}
\end{align*} 
so that
\begin{align*}
\hspace{.5in}&\hspace{-.5in}KL( \tau \nu(S) | \tau \nu^S (S) ) = \sum_{w \in  \perp \cup S} \P_\nu(W=w) \log( \frac{\P_\nu(W=w)}{\P_{\nu^S}(W=w)}) \\
&= \P_\nu(W=\perp)\log( \frac{\P_\nu(W=\perp)}{\P_{\nu^S}(W=\perp)}) + \sum_{i \in S} \frac{1}{k} \P_\nu(W \neq \perp) \log( \frac{\P_\nu(W \neq \perp)}{\P_{\nu^S}(W \neq \perp)}) \\
&= KL\big( \P_\nu(\max_{i\in S} X_i=1) \big| \P_{\nu^S}(\max_{i\in S} X_i=1) \big).
\end{align*}
Thus, KL divergence for marked-bandit feedback is equal to that of simple bandit feedback.  

\noindent\textbf{Semi-Bandit feedback}: 

Let $P$ denote the law of the entire construction for independent distribution, and $Q$ the law of the construction for the distribution. The strategy is to upper bound the KL of $X$, together with the additional information from the hidden variables $Z_2,\dots,Z_k$. In this section, given $v \in \{0,1\}^k$, we use the compact notation $v^{(2;k)}$ to denote the vector $v_2,\dots,v_k$. We can upper bound the KL by 
\begin{eqnarray*}
KL(p(X),Q(X)) &\le& KL(P(X,Z^{(2;k)}),Q(X,Z^{(2;k)}))\\
&=& \sum_{x \in \{0,1\}^k,z^{(2;k)}\in \{0,1\}^{k-1}} P\left(X = x, Z^{(2;k)} = z^{(2;k)}\right) \log \left(\dfrac{P(X = x, Z^{(2;k)} = z^{(2;k)})}{Q(X = x, Z^{(2;k)} = z^{(2;k)})}\right) 
\end{eqnarray*}
By the law of total probability, the above is just
\begin{eqnarray*}
&& \sum_{x^{(2;k)} \in \{0,1\}^{2;k},z^{(2;k)}\in \{0,1\}^{k-1}} P\left(X^{(2;k)} = x^{(2;k)}, Z^{(2;k)} = z^{(2;k)} \right)  \\
&\times& \left(\sum_{x_1 \in \{0,1\}} P\left(X_1 = x_1 \big{|} X^{(2;k)} = x^{(2;k)}, Z^{(2;k)} = z^{(2;k)}\right) \log \left(\frac{P(X = x, Z^{(2;k)} = z^{(2;k)})}{Q(X = x, Z^{(2;k)} = z^{(2;k)})}\right)\right)
\end{eqnarray*}
Again, by the law of total probability, we have
\begin{eqnarray*}
&& \frac{P(X = x, Z^{(2;k)} = z^{(2;k)})}{Q(X = x, Z^{(2;k)} = z^{(2;k)}} \\
&=& \frac{P(X_1 = x_1 \big{|} X^{(2;k)} = x^{(2;k)}, Z^{(2;k)} = z^{(2;k)})}{Q(X_1 = x_1 \big{|} X^{(2;k)} = x^{(2;k)}, Z^{(2;k)} = z^{(2;k)}))} \times \frac{P(X^{(2;k)} = x^{(2;k)}, Z^{(2;k)} = z^{(2;k)})}{Q(X^{(2;k)} = x^{(2;k)}, Z^{(2;k)} = z^{(2;k)}))}
\end{eqnarray*}
Under our construction, $(X_2,\dots,X_k,Z_2,\dots,Z_k)$ have the same joint distribution under either $P$ or $Q$, so the second multiplicand in the second line in the above display is just $1$. Under the law $P$, $X_1$ is independent of $X_2,\dots,X_k,Z_2,\dots,Z_k$, so $P(X_1 = x_1 \big{|} X^{(2;k)} = x^{(2;k)}, Z^{(2;k)} = z^{(2;k)}) = P(X_1 = x_1)$. Under the dependent law $Q$, $X_1$ only depends on $X_2,\dots,X_k,Z_2,\dots,Z_k$ through $W(Z^{(2;k)}) := 1 \oplus_{i=2}^k Z_i \in \{0,1\}$. Hence, if we define the conditional $KL$'s: 
\begin{eqnarray*}
\mathrm{KL}_1 := KL\left(P(X_1),Q(X_1) \big{|} W(z^{(2;k)}) = 1 \right) = \sum_{x_1 \in \{0,1\}} p(X_1 = x_1) \log \left(\frac{P(X_1 = x_1)}{Q(X_1 = x_1 \big{|} W(z^{(2;k)}) = 1)} \right)
\end{eqnarray*}
and define $\mathrm{KL}_0 := KL\left(P(X_1),Q(X_1) \big{|} W(z^{(2;k)}) = 0\right)$ analogously, then 
\begin{eqnarray*}
&&\sum_{x_1 \in \{0,1\}} P(X_1 = x_1 \big{|} X^{(2;k)} = x^{(2;k)}, Z^{(2;k)} = z^{(2;k)}) \log (\frac{P(X = x, Z^{(2;k)} = z^{(2;k)})}{Q(X = x, Z^{(2;k)} = z^{(2;k)})})\\
 &=& \I\left(W(z^{(2;k)}) = 1\right)\mathrm{KL}_1 + \I\left(W(z^{(2;k)}) = 0\right)\mathrm{KL}_0
\end{eqnarray*}
Putting these pieces together,
\begin{eqnarray*}
KL(P(X,Z^{(2;k)}),Q(X,Z^{(2lk)})) &=&  \sum_{(x^{(2;k)},z^{(2;k)} \in \{0,1\}^{2(k-1)}}\I(W(Z^{(2;k)}) = 1)\mathrm{KL}_1 + \I(W(Z^{(2;k)}) = 0)\mathrm{KL}_0  \\
&=& \Pr(W(Z^{(2;k)}) = 1 )\mathrm{KL}_1 + \Pr(W(Z^{(2;k)}) = 0 )\mathrm{KL}_0\\
&=& \frac{1}{2}(\mathrm{KL}_1 + \mathrm{KL}_0)
\end{eqnarray*}
where the last line follows the parity $W(Z^{(2;k)})$ is Bernoulli $1/2$. A straightforward computation bounds $\mathrm{KL}_1$ and $\mathrm{KL}_0$.

\begin{claim} [Bound on $\mathrm{KL}_1$, $\mathrm{KL}_0$ ]
Let $\mathrm{KL}_0$ and $\mathrm{KL}_1$ be defined as above. Then $\mathrm{KL}_0 \leq \frac{p^2 \mu/2}{(1-p)(1-\mu(1-p))}$ and $\mathrm{KL}_1 \leq  \frac{p^2 \mu /2}{1-\mu(1+p)}$. 
\end{claim}
\begin{proof} 
Note $P(X_1 = 1) = \mu$, 
\begin{align*}
Q(X_1 = 1 \big{|} W(z^{(2;k)}) = 0) &= Q(X_1 = 1 \big{|} W(z^{(2;k)}) = 1, Y=1) p + Q(X_1 = 1 \big{|} W(z^{(2;k)}) = 1, Y=0) (1-p) \\ 
&= 0 \cdot p + \mu (1-p) = \mu(1-p)
\end{align*}
and 
\begin{align*}
Q(X_1 = 1 \big{|} W(z^{(2;k)}) = 1) &= Q(X_1 = 1 \big{|} W(z^{(2;k)}) = 1, Y=1) p + Q(X_1 = 1 \big{|} W(z^{(2;k)}) = 1, Y=0) (1-p) \\ 
&= 2\mu p + \mu (1-p) = \mu(1+p).
\end{align*}
Thus, by Lemma~\ref{bernoulli_kl_bounds} we have
\begin{align*}
\mathrm{KL}_0 &= \sum_{x_1 \in \{0,1\}} P(X_1 = x_1) \log \left(\frac{P(X_1 = x_1)}{Q(X_1 = x_1 \big{|} W(z^{(2;k)}) = 0)} \right) \\
&= d( \mu , \mu(1-p)) \leq \frac{(p \mu)^2/2}{\mu(1-p)(1-\mu(1-p))} = \frac{p^2 \mu/2}{(1-p)(1-\mu(1-p))}.
\end{align*}
and
\begin{align*}
\mathrm{KL}_1 &= \sum_{x_1 \in \{0,1\}} P(X_1 = x_1) \log \left(\frac{P(X_1 = x_1)}{Q(X_1 = x_1 \big{|} W(z^{(2;k)}) = 1)} \right) \\
&= d(\mu, \mu(1+p)) \leq \frac{(p \mu)^2/2}{\min\{\mu(1-\mu),\mu(1+p)(1-\mu(1+p))\}} \leq \frac{p^2 \mu /2}{1-\mu(1+p)}.
\end{align*}
\end{proof}

\begin{rem}Despite our seemingly arbitrary construction of random variables in Theorem~\ref{lower_bandit_identification} to produce the resulting measure $\nu$, Theorem~\ref{LowerBoundConverse} states that the joint distribution is unique and would be arrived at using any other construction that satisfied the same properties.
\end{rem}

\begin{rem}[An Upper Bound When $\mu = 1/2$]\label{rem:semibandit_tightness} Suppose that $\mu = 1/2$. Then, our construction implies $Z_i = X_i$ for $i \ge 2$, and thus our bound on the $KL$ is exact. In fact, we can use a simple parity estimator $W(S) = \oplus_{i \in S}X_i$ to distinguish between a subset $S$ of correlated and uncorrelated arms. When $S$ is an independent set, $W(S) \sim \text{Bernoulli}(1/2)$. However, a simple computation reveals that $W(S^*) \sim \text{Bernoulli}(1/2 + p/2)$. Thus, using a parity estimator reduces our problem to finding one coin with bias $p/2$ in a bag of $\binom{n}{k}$ unbiased coins, whose difficulty exactly matches our problem

Surprisingly, Theorem~\ref{LowerBoundConverse} tells us that the construction outlined in this lower bound is the \textbf{unique} construction which yields $k-1$-wise independent marginals of mean $\mu = 1/2$, with gap $p2^{-k}$; in other words, in any $k-1$-wise independent construction with $\mu = 1/2$, the parity estimator is optimal. 
\end{rem}

\subsection{Proof of Lemma~\ref{bernoulli_kl_bounds}}
\begin{proof}[Proof of Lemma~\ref{bernoulli_kl_bounds}]
	If $f(z) = d(z,y)$ then $f'(z) = \log(\frac{z}{1-z}) - \log(\frac{y}{1-y})$, and $f''(z) = \frac{1}{z(1-z)}$ so 
	\begin{align*}
	 2(y-x)^2 \leq  \frac{(y-x)^2/2}{ \sup_{z \in [x,y]} z(1-z)} \leq  d(x,y) \leq \frac{(y-x)^2/2}{ \inf_{z \in [x,y]} z(1-z)}.
	\end{align*}
	If $\epsilon = y-x$ then
	\begin{align*}
	 \inf_{z \in [x,y]} z(1-z) = \inf_{\epsilon \in [0,y-x]} x(1-x) + \epsilon(1-2x) = x(1-x) - [(y-x)(2x-1)]_+
	\end{align*}
\end{proof}

\section{Proof of Lower Bound Converse}

To prove the above proposition, we need a convenient way of describing all feasible probability distributions over $\{0,1\}^k$ which are specified on their $k-1$ marginals. To this end, we introduce the following notation: We shall find it convenient to index the entries of vectors $w \in \R^{k-1}$ by binary strings $t \in \{0,1\}^{k-1}$. At times, we shall need to ``insert'' indices into strings of length $k-2$, as follows: For $u \in \{0,1\}^{k-2}$ and $j \in [k-1]$, denote by $u \oplus_j 0$ the string in $\{0,1\}^{k-1}$ obtained by inserting a $0$ in the $j$-th position of $u$.  We define $u \oplus_j 1$ similarly.

\begin{lem}\label{HelperLem} 
	Let $\Pr_0$ be any distribution over $\{0,1\}^k$. Then, a probability distribution $\Pr$ agrees with $\Pr_0$ on their $k-1$ marginals if and only if, for all binary strings $t \in \{0,1\}^{k-1}$, $\Pr$ is given by
	\begin{eqnarray}
	\Pr(X_{-k} = t,X_k = 0) = w(t) 
	\end{eqnarray}
	where $w \in \R^{2^{k-1}}$ satisfies the following linear constraints:
	\begin{eqnarray*}
	  \forall t \in \{0,1\}^{k-1}: && 0 \le w(t) \le \Pr_0(X_{-k} = t)\\
	 \forall j \in [k-1], u \in \{0,1\}^{k-2} && w(u \oplus_j 0 ) + w(u \oplus_j 1) = \Pr_0( X_{-\{j,k\}} = u_{-j}, X_k = 0) \\
	\end{eqnarray*}
\end{lem}
\begin{rem}
Note that the above lemma makes no assumptions about $k-1$ independence, only that the $k-1$ marginals are constrained 
\end{rem}

\begin{proof}[Proof of Theorem~\ref{LowerBoundConverse}]
Let $\Pr_0$ denote the product measure on $X_1,\dots,X_k$, and $\Pr$ denote our coupled distribution. Fix $\mu \in [0,1]$. For $p \in \{0,1,\dots,k-1\}$, define the probability mass function 
\begin{eqnarray}
\psi(p) := \mu^p (1-\mu)^{k-1-p}
\end{eqnarray}
Further, for $u$ and $t$ in $\{0,1\}^{k-2}$ and $\{0,1\}^{k-1}$, respectively, define the hamming weights $H(t) = \sum_i t_i$ and $H(u) = \sum_i u_i$. 

Since our distribution is $k-1$ wise independent, and each entry $X_i$ has mean $\mu$, we have $\Pr(X_{-k} = t) = \Pr_0(X_{-k} = t) = \psi(H(t))$. Moreover, 
\begin{eqnarray*}
\Pr_0(X_{-\{j,k\}} = u_{-j}, X_k = 0) &=& (1-\mu)\Pr_0(X_{-\{j,k\}} = u_{-j})\\
&=& (1-\mu)\mu^{H(u)}(1-\mu)^{k-2 - H(u)} \\
&=& \mu^{H(u)}(1-\mu)^{k-1-H(u)} = \psi(H(u))
\end{eqnarray*}
Thus, our feasibility set is precisely
\begin{equation}
\begin{aligned}
	  \forall t \in \{0,1\}^{k-1}: && 0 \le w(t) \le \psi(H(t))\\
	 \forall j \in [k-1], u \in \{0,1\}^{k-2} && w(u \oplus_j 0 ) + w(u \oplus_j 1) = \psi(H(u)) \\
\end{aligned}
\end{equation}

\noindent The equality constraints show there is only one degree of freedom, which we encode into $w(\mathbf{0})$:
	\begin{claim}\label{BigPhiClaim} $w$ satisfies the equality constraints of the LP if and only if, for  all $t \in \{0,1\}^{k-1}$ of weight $H(t) = p$, 
	\begin{eqnarray}
	w(t) = \left(-1\right)^{p} w\left(\mathbf{0}\right) + \left(-1\right)^{p-1}\Phi\left(p\right)
	\end{eqnarray}
	where $\Phi(p) = \sum_{i = 0}^{p-1} (-1)^{i} \psi(i)$, so that $\Phi(0)=0$. Note that $\Phi$ satisfies the identity
	\begin{eqnarray}\label{RecursePhi}
	\Phi(p) = \left(-1\right)^{p-1}\psi\left(p-1\right) + \Phi\left(p-1\right)
	\end{eqnarray}
	\end{claim}
	
	\noindent Hence, we can replace the equality constraints by the explicit definitions of $w(t)$ in terms of $w(\mathbf{0})$ and $\Phi(p)$. This leads to the next claim:
	\begin{claim}\label{EqualityClaim}
	$w$ is feasible precisely when
	\begin{eqnarray}
	\max_{0\le p \le k \text{ even }} \Phi(p) \le w(\mathbf{0}) \le \min_{1\le p \le k \text{ odd }} \Phi(p) 
	\end{eqnarray}
	\end{claim}
	
	\noindent We now establish a closed form solution for $\Phi(p)$ when $\mu < 1/2$, and parity-wise monotonicity when $\mu \ge 1/2$:
	\begin{claim}\label{ClaimPhiComp}
	If $\mu < 1/2$, we have $\Phi(p) = (1-\mu)^{k} \left(1 - (\frac{-\mu}{1-\mu})^{p}\right)$, so $\Phi(p)$ is decreasing for odd $p$ and increasing for even $p$. If $\mu \ge 1/2$, $\Phi(p)$ is nondecreasing for odd $p$ and nonincreasing for even $p$
	\end{claim}
	
	To conclude, we note that when $\mu \ge 1/2$, the fact that  $\Phi(p)$ is nondecreasing for odd $p$ and nonincreasing for even $p$ implies that 
	\begin{eqnarray*}
		\max_{0 \le p \le k \text{ even }} \Phi(p) \le w(\mathbf{0}) \le \min_{1 \le p \le k \text{ odd }} \Phi(p) &\iff& \Phi(0) \le w(\mathbf{0}) \le \Phi(1) \\
		 &\iff& 0 \le w(\mathbf{0}) \le \psi(0) \\
		  &\iff& 0 \le w(\mathbf{0}) \le (1-\mu)^{k-1} 
	\end{eqnarray*}
	When $\mu < 1/2$, the fact that $\Phi(p)$ is decreasing for odd $p$ and increasing for even $p$ implies that 
	\begin{eqnarray*}
		&&\max_{0 \le p \le k \text{ even }} \Phi(p) \le w(\mathbf{0}) \le \min_{1 \le p \le k \text{ odd }} \Phi(p) \iff \Phi(k_{odd}) \le w(\mathbf{0}) \le \Phi(k_{even}) \\
		 &\iff& (1-\mu)^k \left(1 - \left(\frac{\mu}{1-\mu}\right)^{k_{even}}\right) \le w(\mathbf{0}) \le (1-\mu)^k \left(1 + \left(\frac{\mu}{1-\mu}\right)^{k_{odd}}\right)
	\end{eqnarray*}
	Since $w(\mathbf{0}) = \Pr(X_1,\dots,X_k = 0)$, we are done.

\end{proof}

\subsection{Proofs}

\begin{proof}[Proof Of Lemma~\ref{HelperLem}]
	We can consider the joint distribution of $(X_1,\dots, X_k)$ as a vector in the $2^{k}$ simplex. However, there are many constraints: in particular, the joint distribution of $X_1,\dots,X_{k-1}$ is entirely determined by the $k-1$-marginals of the distribution. In fact, if $\Pr$ is a distribution over $\{0,1\}^k$, then it must satisfy 
	\begin{eqnarray*}
	\Pr(X_{-k} = t_{-k}, X_k = 1) + \Pr(X_{-k} = t_{-k}, X_k = 0) &=& \Pr(X_{-k} = t_{-k}).
	\end{eqnarray*}
	Hence, without any loss of generality, we may encode any arbitrary probability distribution on $\{0,1\}^k$ by 
	\begin{eqnarray}
	\Pr (X = t) := \begin{cases} w(t_{-k}) & t_k = 0 \\
	\Pr_0(X_{-k} = t_{-k}) - w(t_{-k}) & t_k = 1
	\end{cases}
	\end{eqnarray}
	for a suitable $w \in \R^{2^{k-1}}$. This defines $\Pr$ on the atomic events $\{X = t\}$, and we extend $\Pr$ to all further events by additivity. We now show that the constraints on the Lemma hold if and only if $w$ induces a proper probability distribution $\Pr$ whose $k-1$ marginals coincide with $\Pr$. 

	Recall that $\Pr$ is a proper distribution if and only if it is nonnegative, normalized to one, monotonic, and additive\footnote{ As $X$ has finite support, we don't need to worry about such technical conditions as $\sigma$-additivity}. $\Pr$ satisfies additivity by construction. Moreover, by definition $\sum_{t \in \{0,1\}^{k}} \Pr(X = t) = \sum_{t_{-k} \in \{0,1\}^{k-1}} \Pr_0(X_{-k} = t_{-k}) = 1$, so $\Pr$ is normalized. Finally, monotonicity will follow as long as we establish non-negativity of $\Pr$ on the atomic events $\{X = t\}$. But the constraint that $\Pr (X = t)$ is nonnegative holds if and only if
	\begin{eqnarray}
	0 \le w(t_1,\dots,t_{k-1}) \le \Pr_0(X_{-k} = t_{-k}).
	\end{eqnarray}
	On the other hand, the constraint that $\Pr$'s $k-1$ marginals coincide with $\Pr_0$ is simply that 
	\begin{eqnarray*}
	&&w(t_1,\dots,t_{j-1},0,t_{j+1},\dots,t_{k-1}) + w(t_1,\dots,t_{j-1},1,t_{j+1},\dots,t_{k-1}) \\
	&=& \Pr_0(X_1 = t_1,\dots, X_{j-1} = t_{j-1}, X_{j+1} = t_{j+1},\dots, X_{k-1}= t_{k-1}, X_k = 0)
	\end{eqnarray*}
	which can be expressed more succinctly using the concatenation notation $w(u \oplus_j 0 ) + w(u \oplus_j 1) = \Pr_0( X_{-\{j,k\}} = u_{-j}, X_k = 0) $.
\end{proof}

\begin{proof}[Proof of Claim~\ref{BigPhiClaim}]
		First, we prove ``only if'' by induction on $H(t)$. For $H(t) = 0$, the claim holds since $\Phi(0) = 0$. For a general  $t \in \{0,1\}^{k-1}$ such that $H(t) = p \ge 1$, we can construct a sequence $t_0,\dots,t_{p} \in \{0,1\}^{k-1}$ such that $t_0 = 0$, $t_{p} = t$, and each string $t_s$ is obtained by ``flipping on'' a zero in the string $t_{s-1}$ to 1, that is, there is a string $u_s \in \{0,1\}^{k-2}$ such that $t_s = u_s \oplus_{j_{s}} 1$ and $t_{s-1} = u_s \oplus_{j_s} 0$. Thus, our equality constraints imply that
		\begin{eqnarray*}
		w(t_{p-1}) + w(t) &=& w(t_{p-1}) + w(t_p) \\
		&=& w(u_s \oplus_{j_{s}} 1) + w(u_s \oplus_{j_s} 0)  \\
		&=& \psi(H(u)) = \psi(p-1)	.
		\end{eqnarray*}
		Hence, we get the recursion $w(t_{p}) = \psi(p-1) - w(t_{p-1})$, which by the inductive hypothesis on $t_{p - 1}$ and Equation~\ref{RecursePhi} imply that
		\begin{eqnarray*}
		w(t) &=& \psi(p-1) - \left((-1)^{p-1} w(\mathbf{0}) + (-1)^{p-2}\Phi(p-1)\right)\\
		&=& (-1)^{p}w(\mathbf{0}) + (-1)^{p-1}\Phi(p-1) + \psi(p-1)\\
		&=& (-1)^{p}w(\mathbf{0}) + (-1)^{p-1}(\Phi(p-1) + (-1)^{p-1} \psi(p-1))\\
		&=&  (-1)^{p}  w(\mathbf{0}) + (-1)^{p-1}\Phi(p)
		\end{eqnarray*}
		as needed. Next, we prove the ``if'' direction. Let $u \in \{0,1\}^{k-2}$ have weight $p$. Then
		\begin{eqnarray*}
		w(u \oplus 0) + w(u \oplus 1) &=& (-1)^{p}  w(\mathbf{0}) + (-1)^{p-1}\Phi(p) + (-1)^{p+1}  w(\mathbf{0}) + (-1)^{p}\Phi(p+1)\\
		&=& (-1)^{p-1} \Phi(p) + (-1)^{p}\Phi(p+1)\\
		&=& (-1)^{p-1} \Phi(p) + (-1)^{p}\left(\Phi(p+1 - 1) + (-1)^{p}\psi(p+1 - 1)\right)\\
		&=& \left((-1)^{p-1} +(-1)^{p}\right) \Phi(p) + (-1)^{2p}\psi(p) = \psi(p)
		\end{eqnarray*}
		as needed. 
	\end{proof}

\begin{proof}[Proof of Claim~\ref{EqualityClaim}]
		Our feasibility set is precisely is the set of $w(\mathbf{0})$ such that $0 \le w(\mathbf{0}) \le \psi(0)$, and for all $p \in \{1,2,\dots,k-1\}$
		\begin{eqnarray}
		0 \le (-1)^{p} w(\mathbf{0}) + (-1)^{p-1}\Phi(p) \le \psi(p).
		\end{eqnarray}
		Suppose first that $p$ is even. If $p$ is greater than $1$, then the above constraint together with Claim~\ref{BigPhiClaim} imply
		\begin{eqnarray*}
		\Phi(p) \le w(\mathbf{0}) \le \psi(p) + \Phi(p) = (-1)^{(p+1) - 1}\psi(p) + \Phi(p) = \Phi(p+1).
		\end{eqnarray*}
		If $p$ is $0$, then $\Phi(0) = 0$ and $\Phi(1) = \psi(0)$, so the constraint $0 \le w(\mathbf{0}) \le \psi(0)$ is equivalent to $\Phi(p) \le  w(\mathbf{0}) \le  \Phi(p+1)$ for $p=0$.

		On the other hand, when $p$ is odd, we have $w(\mathbf{0}) \le \Phi(p)$, whilst
		\begin{eqnarray}
		 w(\mathbf{0}) \ge \Phi(p) - \psi(p) = \Phi(p) + (-1)^{(p+1)-1}\psi(p) = \Phi(p+1).
		\end{eqnarray}
		In other words, $w(\mathbf{0}) \le \Phi(p)$ for all $p$ which are either odd and between $1$ and $k-1$, or $p$ of the form $p = q+1$ where $q$ is even and between $1$ and $k-1$. This is precisely the set of all odd $p$ in $1,\dots,k$. By the same token, $w(\mathbf{0}) \ge \Phi(p)$ for all even $p$ in $\{1,\dots,k\}$. Taking the intersection of these lower and upper bounds on $w(\mathbf{0})$ yields
		\begin{eqnarray}
		\max_{0 \le p \le k \text{ even }} \Phi(p) \le w(\mathbf{0}) \le \min_{1 \le p \le k \text{ odd }} \Phi(p) .
		\end{eqnarray}
	\end{proof}

\begin{proof}[Proof of Claim~\ref{ClaimPhiComp}] Let $\rho = \frac{\mu}{1-\mu}$. We can write $\Phi$ yields as  geometric series 
		\begin{eqnarray*}
		\Phi(p) &=& \sum_{i = 0}^{p-1} (-1)^{i} \psi(i) \\
		&=& \sum_{i = 0}^{p-1} (-1)^{i} \cdot \mu^{i} (1-\mu)^{k-1-i} \\
		&=& (1-\mu)^{k-1}\sum_{i = 0}^{p-1} (-1)^{i} (\frac{\mu}{1-\mu})^{i} \\
		&=& (1-\mu)^{k-1}\sum_{i = 0}^{p-1} (-\rho)^{i}
		\end{eqnarray*}
		When $\mu \ge 1/2$, $\rho \ge 1$, and thus this series is nondecreasing for odd $p$ and nonincreasing for even $p$. When $\rho < 1/2$, the series is decreasing for odd $p$ and increasing for even $p$ and in fact we have
		\begin{eqnarray*}
		\Phi(p)&=& (1-\mu)^{k-1} \frac{1 - (-\rho)^{p}}{1 + \rho }\\
		\Phi(p)&=& (1-\mu)^{k-1} \frac{1 - (-\frac{\mu}{1-\mu})^{p}}{1 + \frac{\mu}{1-\mu}  }\\
		&=& (1-\mu)^{k} \left(1 - (\frac{-\mu}{1-\mu})^{p}\right)
		\end{eqnarray*}
	\end{proof}

\section{Proof of Theorem~\ref{lower_bound_independent}: Lower Bound for Independent Arms}
As in the proof of Theorem~\ref{lower_bandit_identification}, let $\nu(a)$ describe the joint probability distribution of $\nu$ restricted to the set $i \in a$. Note that $\nu(a) = \prod_{i \in a} \nu_i$.  And for any $a \in A$ let $\tau \nu(a)$ represent the Bernoulli probability distribution describing $\max_{i \in a} X_i$ under distribution $\nu$. Let $\epsilon>0$. For each $j \in [n]$ let $\nu^j$ be a product distirbution of Bernoullis fully defined by its marginals $\mu_i^j := \E_{\nu_i^j}[X_i]$ and
\begin{align*}
\mu_i^j = \begin{cases}  \mu_k+\epsilon & \text{if $i =j$ and $i > k$} \\ 
			 	    \mu_{k+1}-\epsilon & \text{if $i=j$ and $i \leq k$} \\
				    \mu_i & \text{if $i \neq j$.}\end{cases} 
\end{align*}
By Lemma 1 of \cite{kaufmann2014complexity}, for every $j \in [n]$
\begin{align*}
\sum_{a \in \binom{[n]}{p}} \E_\nu[ T_{a} ]  KL( \tau \nu(a) |  \tau \nu^j (a) ) \geq \log( \tfrac{1}{2\delta} ),
\end{align*}
for arbitrarily small $\epsilon$, so in what follows let $\epsilon=0$. Then
\begin{align*}
 KL( \tau \nu(a) | \tau \nu^j (a) ) = \begin{cases} 0 & \text{if $j \notin a$} \\ 
 							     d\left( (1-\mu_j) \prod_{i \in a \setminus j} (1-\mu_i)  | (1-\mu_j - \Delta_j )  \prod_{i \in a \setminus j} (1-\mu_i)  \right) & \text{if $j \in a$ and $j > k$} \\
							     d\left( (1-\mu_j) \prod_{i \in a \setminus j} (1-\mu_i)  | (1-\mu_j + \Delta_j )  \prod_{i \in a \setminus j} (1-\mu_i)  \right)  & \text{if $j \in a$ and $j \leq k$} \end{cases}
\end{align*}
where for $j > k$, by invoking Lemma~\ref{bernoulli_kl_bounds},
\begin{align*}
d\left( (1-\mu_j) \prod_{i \in a \setminus j} (1-\mu_i)  | (1-\mu_j - \Delta_j )  \prod_{i \in a \setminus j} (1-\mu_i)  \right) \\ 
\leq \frac{ \Delta_j^2  \left( \prod_{i \in a \setminus j} (1-\mu_i)\right)^2 }{2 \left( 1 - (1-\mu_j)\prod_{i \in a \setminus j} (1-\mu_i) \right)  \left( (1-\mu_j - \Delta_j)\prod_{i \in a \setminus j} (1-\mu_i)\right)} \\
\leq \frac{ \Delta_j^2  \left( \prod_{i \in a \setminus j} (1-\mu_i)\right) }{ 2 \left( 1 - (1-\mu_j)\prod_{i \in a \setminus j} (1-\mu_i) \right)  (1-\mu_j- \Delta_j)}
\end{align*}
and a similar bounds holds for $j \leq k$. If $h_j = \max_{a \in \binom{[n]-j}{p-1}} \prod_{i \in a \setminus j} (1-\mu_i)$ and 
\begin{align*}
\tau_j = \begin{cases} \frac{(1-\mu_j-\Delta_j)}{\Delta_j^2} \frac{1-(1-\mu_j)h_j}{h_j} & \text{if $j > k$} \\ 
 											     \frac{(1-\mu_j)}{\Delta_j^2} \frac{1-(1-\mu_j+\Delta_j)h_j}{h_j} & \text{if $j \leq k$}
									 \end{cases}
\end{align*}
$\forall j \in [n]$ then
\begin{align} \label{j_in_rounds}
\sum_{a \in \binom{[n]}{p}: j \in a} \E_\nu[ T_{a} ]  \geq  2 \tau_j \log(  \tfrac{1}{2\delta} ) 
\end{align}
or, in words, arm $j$ must be included in a number of bandit observations that is at least the right-hand-side of \eqref{j_in_rounds}. Because $p$ arms can be selected per evaluation, if we assume perfect divisibility to minimize the number of evaluations, then we conclude that 
\begin{align}\label{lower_tau}
\sum_{a \in \binom{[n]}{p}} \E_\nu[ T_{a} ]   \geq 2 \log(  \tfrac{1}{2\delta} ) \max\left\{ \max_{j = 1,\dots,n} \tau_j , \frac{1}{p} \sum_{j=1}^n \tau_j \right\} \geq \log(  \tfrac{1}{2\delta} ) \left( \max_{j = 1,\dots,n} \tau_j + \frac{1}{p} \sum_{j=1}^n \tau_j \right)
\end{align}
where the first argument of the $\max$ follows from the fact that the number of rounds must exceed the number of bandit evaluations each arm must be included in.

For semi-bandit feedback, we use the same $\nu^j$ construction but now realize that 
\begin{align*}
 KL( \nu(a) |  \nu^j (a) ) = \begin{cases} 0 & \text{if $j \notin a$} \\ 
 							     d\left( 1-\mu_j | 1-\mu_j - \Delta_j  \right) & \text{if $j \in a$ and $j > k$} \\
							     d\left( 1-\mu_j   |  1-\mu_j + \Delta_j   \right)  & \text{if $j \in a$ and $j \leq k$}. \end{cases}
\end{align*}
Using the same series of steps as above, we find that if
\begin{align*}
\tau_j =  \begin{cases} \frac{\mu_j (1-\mu_j-\Delta_j)}{\Delta_j^2} & \text{if $j > k$} \\ 
 											     \frac{(\mu_j-\Delta_j) (1-\mu_j)}{\Delta_j^2}& \text{if $j \leq k$}
									 \end{cases}
\end{align*}
then \eqref{lower_tau} holds with these defined values of $\tau_j$ for the semi-bandit case.

\end{document}